\documentclass{article}                     
\usepackage{graphicx}

\usepackage{amsthm,amsmath,natbib}
\usepackage{natbib}
\RequirePackage[colorlinks,citecolor=blue,urlcolor=blue,hyperfootnotes=false]{hyperref}

\usepackage{latexsym}

\usepackage[dvipsnames]{xcolor}
\definecolor{Bleu}{RGB}{0,0,204}
\definecolor{Violet}{RGB}{102,0,204}
\definecolor{Rouge}{RGB}{204,0,0}
\definecolor{Highlight}{RGB}{251,0,0}
\definecolor{awesome}{rgb}{1.0, 0.13, 0.32}
\usepackage{zref-xr,zref-user} % used to reference between documents
\usepackage{titletoc}

\usepackage{amsmath}
\usepackage{amssymb} % \mathbb
\usepackage{booktabs}
\usepackage{afterpage} % Add footnote below algorithm
\usepackage{dsfont} % For indicators

\usepackage{rotating}
\usepackage{enumitem}

\usepackage[mathscr]{eucal} % $\mathscr$

\usepackage{bm}

\usepackage{algorithm}
\usepackage[noend]{algpseudocode}

\newtheorem{theorem}{Theorem}

\newtheorem{lemma}[theorem]{Lemma}
\newtheorem{proposition}[theorem]{Proposition}

\theoremstyle{definition}

\newcommand{\openr}{\hbox{${\rm I\kern-.2em R}$}}
\newcommand{\openn}{\hbox{${\rm I\kern-.2em N}$}}

\DeclareMathOperator{\argmax}{argmax}
\DeclareMathOperator{\argmin}{argmin}

\DeclareMathOperator{\KL}{KL}

\DeclareMathOperator{\Kinf}{\mathcal{K}_{\textnormal{inf}}}
\DeclareMathOperator{\E}{\mathbb{E}}
\DeclareMathOperator{\Prob}{\mathbb{P}}

\DeclareMathOperator{\A}{\hat{\mathcal{A}}}
\DeclareMathOperator{\Supp}{Support}
\DeclareMathOperator{\Ind}{\mathds{1}}
\newcommand{\cc}{c}

\def\cV{\mathcal{V}}
\def\cA{\mathcal{A}}
\def\cF{\mathcal{F}}
\def\cD{\mathcal{D}}
\def\cL{\mathcal{L}}
\def\cM{\mathcal{M}}
\def\cN{\mathcal{N}}
\def\cS{\mathcal{S}}
\def\cU{\mathcal{U}}
\def \ind{\mathds{1}}
\newcommand{\bE}{\mathbb{E}}
\newcommand{\bP}{\mathbb{P}}
\def \Reg {\mathrm{Regret}}

\allowdisplaybreaks

\usepackage[normalem]{ulem}
\usepackage{longtable}
\begin{document}

\title{Asymptotically Optimal Algorithms for Budgeted \\ Multiple Play Bandits}
% \subtitle{Do you have a subtitle?\\ If so, write it here}

\author{Alex Luedtke$^1$, 
        Emilie Kaufmann$^2$   and Antoine Chambaz$^3$ \\        
        \small $^1$University of Washington, Department of Statistics. \\   \small$^2$CNRS \& Univ. Lille,  CRIStAL (UMR 9189), Inria Lille. \\ \small $^3$Universit\'{e} Paris Descartes, Laboratoire MAP5.}          
\date{}
        
\maketitle

\begin{abstract}
We study  a generalization of  the multi-armed bandit problem  with multiple
plays where there is  a cost associated with pulling each  arm and the agent
has a budget at each time that dictates how much she
can expect  to spend.  We  derive an asymptotic  regret lower bound  for any
uniformly efficient  algorithm in our  setting. We  then study a  variant of
Thompson sampling  for Bernoulli rewards  and a  variant of KL-UCB  for both
single-parameter  exponential  families   and  bounded,  finitely  supported
rewards. We show  these algorithms are asymptotically optimal,  both in rate
and  leading  problem-dependent constants,  including  in  the thick  margin
setting where multiple arms fall on the decision boundary.
\end{abstract}

% history:
% \received{\smonth{1} \syear{0000}}

%\tableofcontents

% \end{frontmatter}

\section{Introduction}

In the classical multi-armed bandit problem, an agent is repeatedly confronted with a set of $K$ probability distributions $\nu_1,\dots,\nu_K$ called \emph{arms} and must at each round select one of the available arms to pull based on their knowledge from previous rounds of the game. Each played arm presents the agent with a reward drawn from the corresponding distribution, and the agent's objective is to maximize the expected sum of their rewards over time or, equivalently, to minimize the total regret (the expected reward of pulling the optimal arm at every time step minus the expected sum of the rewards corresponding to their selected actions). To play the game well, the agent must balance the need to gather new information about the reward distribution of each arm (exploration) with the need to take advantage of the information that they already have by pulling the arm for which they believe the reward will be the highest (exploitation).

The bandit problem first
started receiving rigorous mathematical attention slightly under a century ago
\citep{Thompson1933}. This early  work focused on Bernoulli  rewards, that are
relevant  in  the  simplest  modeling  of a  sequential  clinical  trial,  and
presented a  Bayesian algorithm  now known as  Thompson sampling.   Since that
time,  many  authors  have  contributed  to  a  deeper  understanding  of  the
multi-armed bandit problem, both with Bernoulli and other reward distributions
and   either    from   a    Bayesian   \citep{Gittins1979}    or   frequentist
\citep{Robbins1952}  perspective. \cite{Lai&Robbins1985}  established a  lower
bound on  the (frequentist) regret of  any algorithm that satisfies  a general
uniform efficiency condition.  This lower  bound provides a concise definition
of  asymptotic  (regret)   optimality  for  an  algorithm:   an  algorithm  is
asymptotically      optimal     when      it      achieves     this      lower
bound. \cite{Lai1987} introduced  what are known  as upper
confidence bound  (UCB) procedures for deciding  which arm to pull  at a given
time step. In short, these procedures compute a UCB for the expected reward of
each arm at each time and pull the  arm with the highest UCB. Many variants of
UCB  algorithms   have  been  proposed   since  then  (see   the  Introduction
of  \citealp{Cappeetal2013}  for  a   thorough  review),  with  more
explicit indices and/or finite-time regret  guarantees.  Among them the KL-UCB
algorithm  \citep{Cappeetal2013} is  proved to  be asymptotically  optimal for
rewards   that   belong   to    a   one-parameter   exponential   family   and
finitely-supported  rewards.   Meanwhile, there  has  been  a recent  interest
in  the  theoretical  understanding   of  the  previously  discussed
Thompson  sampling  algorithm,  whose  first  regret  bound  was  obtained  by
\cite{Agrawal&Goyal2011}. Since then, Thompson Sampling  has been proved to be
asymptotically          optimal         for          Bernoulli         rewards
\citep{Kaufmannetal2012,Agrawal&Goyal2012}   and   for  reward   distributions
belonging to univariate exponential families \citep{Kordaetal2013}.

There has recently been a surge of interest in the multi-armed bandit problem,
due to its applications to (online) sequential content recommendation. In this
context each arm models the feedback  of an agent to a specific item
that can be displayed (e.g. an  advertisement). In this framework, it might be
relevant to display  \emph{several} items at a time, and  some variants of the
classical bandit problems that have  been proposed in the literature
may  be considered.  In  the \emph{multi-armed  bandit  with multiple  plays},
$m \geq 1$  out of $K$ arms are  sampled at each round and  all the associated
rewards   are   observed    by   the   agent,   who    receives   their   sum.
\cite{Anantharametal1987}  present  a regret  lower  bound  for this  problem,
together with a (non-explicit) matching strategy. More explicit strategies can
be  obtained  when  viewing  this  problem  as  a  particular  instance  of  a
\emph{combinatorial bandit problem  with semi-bandit feedback}.  Combinatorial
bandits,   originally  introduced   by   \cite{CesaBianchi&Lugosi2012}  in   a
non-stochastic setting, present the agent  with possibly structured subsets of
arms at  each round: once a  subset is chosen,  the agent receives the  sum of
their   rewards.   The   semi-bandit   feedback  corresponds   to   the   case
where
% when
the  agent  is  able   to  see  the  reward  of  each   of  the  sampled  arms
\citep{Audibertetal2011}.   Several extensions  of  UCB  procedures have  been
proposed      for      the      combinatorial      setting      (see      e.g.
\cite{Chenetal2013,Combesetal2015b}),       with      logarithmic       regret
guarantees. However, existing regret upper bounds  do not match the lower bound
of \cite{Anantharametal1987}.
In  particular,  despite  the  strong practical  performance  of  KL-UCB-based
algorithms in  some combinatorial  settings (including  multiple-plays), their
asymptotic  optimality has  never been  established. Extending  the optimality
result  from the  single-play setting  has proven  challenging, especially  in
settings where the optimal set of $m$ arms in non-unique.
Recently, \cite{Komiyamaetal2015} proved the asymptotic optimality of Thompson
sampling for  multiple-play bandits with  Bernoulli rewards in the  case where
the  arm with  the $m^\textnormal{th}$  largest mean  is unique.  An important
consequence of the uniqueness of  the $m^\textnormal{th}$ largest mean is that
the optimal set of $m$ arms is  necessarily unique, which may not be plausible
in practice.

In  this  paper,  we  extend  the multiple  plays  model  in  two  directions,
incorporating a \emph{budget constraint} and an \emph{indifference
  point}. Given a known cost $c_a$ associated with pulling each arm $a$, at each round a
subset  of arms  $\hat{\cA}(t)$  is selected,  so that  the  expected cost  of
pulling the  chosen arms is  at most the  budget $B$.  More  formally, letting
$C(t)\equiv\sum_{a  \in \hat{\cA}(t)}  c_a$, one  requires $\E[C(t)]  \leq B$,
where the expectation  over the random selection of  the subset $\hat{\cA}(t)$
is taken conditionally on 
past observations.   The agent observes  the reward associated to  the selected
arms          and          receives           a          total          reward
$R(t) =  \sum_{a=1}^K Y_{a}(t) \ind_{(a \in  \hat{\cA}(t))}$, where $Y_{a}(t)$
is drawn from  $\nu_a$.  This reward is  then compared to what  she could have
obtained, had she spent the same budget  on some other activity, for which the
expect reward per cost unit is $\rho \geq 0$ (that is, the agent may prefer to
use that  money for some  purpose that has reward  to cost ratio  greater than
$\rho$ and is external  to the bandit problem). We note that, for positive reward distributions, choosing $\rho=0$ corresponds to taking an action at every round. The agent's  gain at round $t$
is thus defined as
\[G(t) =  R(t) -  \rho C(t)  = \sum_{a  \in \hat{\cA}(t)}\left(Y_{a}(t)  - c_a
    \rho\right).\] The  goal of  the agent  is to  devise a  sequential subset
selection strategy  that maximizes the expected  sum of her gains,  up to some
horizon $T$ and for which the budget constraint $\E[C(t)] \leq B$ is satisfied
at each round $t \leq T$.  In particular, arm $a$ is ``worth'' drawing (in the
sense that it increases the expected gain) only if its average reward per cost
unit, $\mu_a/\cc_a$ (where $\mu_a$ is the expectation of $\nu_a$), is at least
the indifference point $\rho$.

This  new  framework  no  longer  requires  the number  of  arm  draws  to  be
fixed. Rather,  the number  of arm  draws is selected  to exhaust  the budget,
which makes sense  in several online marketing scenarios. One  can imagine for
example a  company targeting a new  market on which  it is willing to  spend a
budget $B$ per  week. Each week, the  company has to decide  which products to
advertise for, and  the cost of the advertising campaign  may vary. After each
week, the income  associated to each campaign $a$ is  measured and compared to
the minimal  income of $\rho  c_a$ that can  be obtained when  targeting other
(known)  markets  or  investing  the   money  in  some  other  well-understood
venture.  Another possible scenario is that the  same item  can be
displayed on several marketplaces never explored before for different costs, and the  seller has to sequentially choose the different places he wants to display the items on while keeping the total budget spend smaller than $B$ and maintaining a profitability larger than what can be obtained on a reference market place with reward per cost unit $\rho$.

Our first contribution  is to characterize the best  attainable performance in
terms  of regret (with respect  to the  gain $G(t)$,  not the  total
  reward $R(t)$) in this multiple-play bandit scenario with cost constraints,
thanks to  a lower bound  that generalizes that  of \cite{Anantharametal1987}.
We then study natural extensions of two existing bandit algorithms (KL-UCB and
Thompson sampling) to  our setting.  We prove both  rate and problem-dependent
leading  constant  optimality for  KL-UCB  and  Thompson sampling.   The  most
difficult part  of the proof is  to show that  the optimal arms away  from the
margin are pulled in almost every  round (specifically, they are pulled in all
but a sub-logarithmic number  of rounds). \cite{Komiyamaetal2015} studied this
problem  for Thompson  sampling  in multiple-play  bandits  using an  argument
different than  that used in this  paper.  We provide a  novel proof technique
that  leverages the  asymptotic lower  bound  on the  number of  draws of  any
suboptimal arm.  While  this lower bound on suboptimal arm  draws is typically
used  to prove  an asymptotic  lower  bound on  the regret  of any  reasonable
algorithm, we use  it as a key  ingredient for our proof  of an asymptotically
optimal \textit{upper  bound} on the  regret of KL-UCB and  Thompson sampling,
i.e. to prove the asymptotic optimality of these two
algorithms. Also, throughout the manuscript, we do not assume that the set of optimal arms
is  unique, unlike  most  of  the existing  work  on (standard)  multiple-play
bandits.

The rest of the article 
is organized  as follows. Section~\ref{sec:methintro} outlines  our problem of
interest.   Section~\ref{sec:lb} provides  an  asymptotic lower  bound on  the
number  of suboptimal  arm draws  and on  the regret.   Section~\ref{sec:algs}
presents the  two sampling algorithms we  consider in this paper  and theorems
establishing  their asymptotic  optimality: KL-UCB  (Section~\ref{sec:kl}) and
Thompson  sampling  (Section~\ref{sec:thom}).  Section~\ref{sec:exp}  presents
numerical     experiments     supporting     our     theoretical     findings.
Section~\ref{sec:proofoutlines}   presents  the   proofs  of   our  asymptotic
optimality  (rate  and  leading  constant) results  for  KL-UCB  and  Thompson
Sampling.  Section~\ref{sec:conc} gives  concluding remarks.  Technical proofs
are postponed to  the Appendix.

\section{Multiple plays bandit with cost constraint} \label{sec:methintro}

We consider a  finite collection of arms $a\in\{1,\ldots,K\}$,  where each arm
has real-valued marginal  reward distribution $\nu_a$ whose mean  we denote by
both $\mu_a$  and $E(\nu_a)$. Each  arm belongs to a  (possibly nonparametric)
class  of  distributions  $\cD$.    We  use  $\mathcal{V}$  to  denote
$(\nu_1,\ldots,\nu_{K})$, where $\mathcal{V}$ belongs to % has
any model $\cD_K$ that is
variation-independent  in  the  sense  that,  for  each  $a\in\{1,\ldots,K\}$,
knowing  the   joint  distribution  of   the  rewards  $a'\not=a$   places  no
restrictions on the collection of  possible marginal distributions of $\nu_a$,
i.e.  $\nu_a$ could be equal to any element in $\cD$. More formally, letting $\cD_{-a}$ denote the collection of joint distributions of the rewards $a'\not=a$ implied by at least one distribution in $\cD_K$, variation independence states that, for each $a\in\{1,\ldots,K\}$ it is true that, for every joint distribution $V_{-a}\in\mathcal{D}_{-a}$ and every distribution $\nu_a\in\cD$, there exists a distribution in $\mathcal{D}_K$ whose joint distribution of the rewards $a'\not=a$ is equal to $V_{-a}$ and whose marginal distribution of reward $a$ is equal to $\nu_a$. An example of a
statistical  model satisfying  this variation-independence  assumption is  the
distribution in which the  rewards of all of the arms  are independent and the
marginal distributions $\nu_a$ fall in  $\cD$ for all $a$, though this
assumption also  allows for high levels  of dependence between the  rewards of
the arms,  i.e. is not  to be confused  with the \textit{much  stronger} model
assumption of independence between the different arms.

\subsection{The sequential decision problem}

Let $\{(Y_1(t),\ldots,Y_K(t))\}_{t=1}^\infty$ be % denote
an independent and identically distributed (i.i.d.)  sample from the
distribution $\cV$. In the multiple-play bandit with cost constraint, each arm
$a$ is associated with a known  \emph{cost} $c_a>0$. The model also depends on
a  known  \emph{budget  per   round}  $B$  and  \emph{indifference  parameter}
$\rho \ge 0$.   At round $t$, the  agent selects a subset $\A(t)$  of arms and
subsequently observes  the action-reward pairs $\{(a,Y_a(t))  : a\in \A(t)\}$.
We emphasize that  the agent is aware that reward  $Y_a(t)$ corresponds to the
action $a\in\A(t)$.  This subset $\A(t)$ is drawn from a distribution $Q(t-1)$
over  $\cS_K$,  the set  of  all  subsets of  $\{1,\dots,K\}$,  that
depends  on the  observations gathered  at the  $(t-1)$ previous  rounds. More
precisely, $Q(t)$ is $\cF(t)$-measurable, where $\cF(t)$ is the $\sigma$-field
generated by all action-reward pairs  seen at times $1,\ldots,t$, and possibly
also  some exogenous  stochastic mechanism.   We  use $q_a(t)$  to denote  the
probability that arm $a$ falls in $\A(t+1)\sim Q(t)$.

Given the  budget $B$  and the  indifference parameter  $\rho$, at  each round
$(t+1)$ the distribution $Q(t)$ must respect the budget constraint
\begin{align}
\E_{\cA\sim Q(t)}\left[\sum_{a\in \cA} \cc_a \right] \le B, \ \ \ \text{or, equivalently,} \ \ \ \sum_{a=1}^{K} \cc_a q_a(t) \leq B.\label{SoftBC}
\end{align}
Upon    selecting     the    arms,    the    agent     receives    a    reward
$R(t+1)   =   \sum_{a    \in   \A(t+1)}   Y_a(t+1)$   and    incurs   a   gain
$G(t+1)  =\sum_{a  \in  \A(t+1)}  (Y_a(t+1) -  c_a\rho)$.  Given  a  (possibly
unknown)  horizon $T$,  the goal  of  the agent  is  to adopt  a strategy  for
sequentially selecting the distributions $Q(t)$ that maximizes
\[\bE\left[\sum_{t=1}^T G(t)\right],\]
while satisfying, at each round $t=0,\dots,T-1$ the budget constraint \eqref{SoftBC}. This constraint may be viewed as a `soft' budget constraint, as it allows the agent to (slightly) exceed the budget at some rounds, as long as the expected cost remains below $B$ at each round. We shall see below that considering a `hard' budget constraint, that is selecting at each round a deterministic subset $\hat{\mathcal{A}}(t)$  that satisfies  $\sum_{a=1}^K \cc_a \ind_{(a \in \hat{\cA}(t))} \leq B$, is a much harder problem. Besides, in the marketing examples described in the introduction, it makes sense to consider a large time horizon and to allow for minor budget crossings.  
Under the  soft budget  constraint \eqref{SoftBC},  if we  knew the  vector of
expected mean rewards  $\bm\mu \equiv (\mu_1,\dots,\mu_K)$, at  each round $t$
we would draw a subset from a distribution %\change{$Q^\star$, defined as an optimal solution to
\begin{equation}Q^\star  \in  \underset{Q}{\text{argmax}}  \
  \bE_{S  \sim  Q}\left[\sum_{a  \in  S}  (\mu_a -  c_a  \rho)\right]  \  \  \
  \text{such that} \  \ \ \bE_{S \sim Q} \left[\sum_{a  \in S} c_a\right] \leq
  B.\label{OriginalPb}\end{equation}
Above, the argmax is over distributions $Q$ with support on the power set of $\{1,\ldots,K\}$. 
Noting that the two
expectations   only  depend   on  the   marginal  probability   of  inclusions
$q_a = \bP_{S \sim Q}\left( a \in S\right)$, it boils down to finding a vector
$\bm q^\star = (q_a)_{a =1}^K$ that  satisfies %\change{is an optimal solution to
\begin{equation}\label{FractionalKnapsack}\bm q^\star \in  \underset{\bm q \in
    [0,1]^K}{\text{argmax}}  \ \sum_{a  =1}^K  q_a(\mu_a -  c_a \rho)  \ \  \
  \text{such that} \ \ \ \sum_{a =1}^K q_a c_a \leq B.\end{equation}
An oracle  strategy would  then draw  $S$ from  a distribution  $Q^\star$ with
marginal  probabilities of  inclusions given  by $\bm  q^\star$ (e.g.  including
independently  each arm  $a$ with  probability $q^\star_a$).  The optimization
problem \eqref{FractionalKnapsack}  is known as a  fractional knapsack problem
\citep{Dantzig1957}, and its solution is  a greedy strategy, that is described
below. It is expressed  in terms of the reward-to-cost ratio  of each arm $a$,
defined as $\rho_a \equiv \mu_a/c_a$.  

\begin{proposition}\label{prop:Oracle}Introduce
\[\rho^\star  \equiv \left\{ \begin{array}{cl}
                             &\rho  \ \text{ if } \ \sum_{a : \rho_a > \rho} c_a < B, \\
                             &\sup \{  r \geq 0  : \sum_{a  : \rho_a >  r} c_a
                               \geq B\} \ge \rho\ \text{ otherwise},\\ 
                            \end{array}
                          \right.\] and define the three sets
\begin{align*}
\mbox{optimal arms away from the margin: }\ \ &\cL\equiv \{a : \rho_a>\rho^{\star}\}, \\
\mbox{arms on the margin: }\ \ &\cM\equiv \{a : \rho_a=\rho^{\star}\}, \\
\mbox{suboptimal arms away from the margin: }\ \ &\cN\equiv \{a : \rho_a<\rho^{\star}\}.
\end{align*}
Then $\bm  q^\star$ is solution  to \eqref{FractionalKnapsack} if and  only if
$q_a^\star = 1$  for all $a \in \cL$,  $q^\star_b =0$ for all $b  \in \cN$ and
$\sum_{a   \in   \cM}   c_aq_a^\star   =   B    -   \sum_{a   \in   \cL}   c_a$   if
$\rho^\star > \rho$.
\end{proposition}

We would like to emphasize that, just like the quantities $Q^\star$, $q^\star$ or $\rho_a$ defined above, the quantity $\rho^\star$ defined in Proposition~\ref{prop:Oracle} depends on the value of $\rho$, the vector of cost and on the vector of means $\bm\mu$. When we need to materialize this dependency in $\bm\mu$ we shall use the notation $\rho^\star(\bm\mu)$, but it is sometimes omitted for the sake of readability. 

From Proposition~\ref{prop:Oracle}, proved in  Appendix~\ref{proofs:Oracle},  the optimal  strategy sorts  the
items  by  decreasing  order  of  $\rho_a$,   and  includes  them  one  by  one
($q^\star_a=1$),  as  long as  the  value  increases  and  the budget  is  not
exceeded. Then  we can
identify two situations: if $\rho^\star(\bm\mu)  = \rho$, there are not enough
interesting items (i.e. such that $\rho_a > \rho$) to saturate the budget, and
the  optimal   strategy  is  to   include  all  the  interesting   items.   If
$\rho^\star(\bm\mu)> \rho$, some probability of  inclusion is further given to
the items on the  margin in order to saturate the  budget constraint.  In that
case,  the  margin is  always  non-empty:  there  exist  items $a$  such  that
$\rho^\star(\bm\mu) = \rho_a$.

\paragraph{Recovering the multiple-play bandit model.}   By choosing $c_a = 1$
for all  arm $a$, $B=m$ and  $\rho=0$, we recover the  classical multiple-play
bandit model.  In  that case $\rho^\star(\bm\mu) = \mu_{[m]}$,  where $[m]$ is
the arm with the $m^{\textnormal{th}}$ largest mean  and $Q^\star = \delta_{\{[1],\dots,[m]\}}$ is a
solution to \eqref{OriginalPb}: the corresponding oracle strategy always plays
the $m$ arms with largest means. % mean

\paragraph{Hard and soft constraints.} 
Under hard budget constraints, if we  knew the vector of expected mean rewards
$\bm\mu$, at each round $t$ we would pick a subset %\change{a subset $S^\star$, where $S^\star$ is an optimal solution to
\begin{equation}S^\star  \in \underset{S  \in \cS_K}{\text{argmax}}  \ \sum_{a
    \in S} (\mu_a - c_a \rho) \ \ \  \text{such that} \ \ \ \sum_{a \in S} c_a
  \leq B.\label{Knapsack}\end{equation}
This is a $0/1$ knapsack problem, that
is much harder  to solve than the above fractional  knapsack problem. In fact,
$0/1$ knapsack problems are NP-hard, though  they are, admittedly, some of the
easiest % easier
problems   in  this   class,  and   reasonable  approximation   schemes  exist
\citep{Karp1972}.   Nonetheless,  the  greedy   strategy  (including  arms  by
decreasing  order of  $\rho_a$ while  the budget  is not  exceeded, with  ties
broken arbitrarily) is not generally  a solution to \eqref{Knapsack}. However,
using Proposition~\ref{prop:Oracle},  one can identify  some examples
  where there exist
deterministic  solutions to  \eqref{FractionalKnapsack}, i.e.   solutions such
that $q_a^\star \in \{0,1\}$ that are therefore solutions to \eqref{Knapsack}:
if  $\rho^\star(\bm\mu)=\rho$  or  if  there  exists $m  \in  \cM$  such  that
$\sum_{a \in \cL\cup \{m\}} c_a = B$. Hence the multiple-play bandit model can
be viewed as a particular instance of the multiple plays model under both hard
or soft budget constraint.  In the rest of the article, % sequel,
we only consider  soft budget constraints, as there is  generally no tractable
oracle under hard budget constraints.

\paragraph{High-probability bound on the budget spent by a finite horizon $T$.} In Appendix~\ref{app:finitehorizon}, we outline how one could analyze the regret of algorithms that respect the soft budget constraint \eqref{SoftBC} at each time $t$ in a finite-horizon problem in which the requirement that \eqref{SoftBC} hold at each time $t$ is replaced by the hard budget constraint that $\sum_{t=1}^T\sum_{a\in \hat{\cA}(t)} \cc_a\le BT$ almost surely. Our argument suggests that the regret in these settings should be no worse than $O(\sqrt{T})$.
 
\subsection{Regret decompositions}\label{sec:regretdecomp}

The best achievable (oracle) performance  consists in choosing, at every round
$t$, $Q(t)$  to be the  optimal distribution $Q^\star$ whose  probabilities of
inclusions are  described in  Proposition~\ref{prop:Oracle}.  Using the definitions introduced in 
Proposition~\ref{prop:Oracle}, such  a strategy
ensures an expected gain at each round of
\begin{align}
G^\star&\equiv \sum_{a=1}^K q_a^\star (\mu_a-c_a\rho). \label{roundgain}
\end{align}
The quantity above is the reward from pulling the chosen arms relative to the reward from reallocating the expected cost of the strategy, namely $\sum_{a=1}^K q_a^\star c_a$, to pursue the action (which is external to the bandit problem) that has reward-to-cost ratio equal to the indifference point $\rho$. We prove the following identity in Appendix~\ref{proofs:Oracle}.
\begin{proposition}\label{prop:Gain}
It holds that
\begin{align*}
G^\star&= \sum_{a \in \cL} \mu_a + \rho^\star\left(B - \sum_{a \in \cL}
    c_a\right)    -   B\rho.
\end{align*}
\end{proposition}

Maximizing the  expected total  gain is equivalent  to minimizing  the regret,
that is the difference in performance compared to the oracle strategy:
\begin{align*}
\Reg(T,\cV,\texttt{Alg})&\equiv T G^\star - \E_\cV\left[\sum_{t=1}^T G(t)\right],
\end{align*}
where the sequence  of gains $G(t)$ is obtained  under algorithm \texttt{Alg}.
The following  statement, proved in Appendix~\ref{proofs:Oracle},  provides an
interesting  decomposition of  the  regret, as  a function  of  the number  of
selections          of          each         arm,          denoted          by
$N_a(T)\equiv \sum_{t=1}^T
\Ind\{a\in\A(t)\}$. 

\begin{proposition}\label{prop:RegretDec}  With  $\rho^\star =  \rho^\star(\bm\mu)$,
  $\cL,\cN$  defined as  in Proposition~\ref{prop:Oracle},  for any  algorithm
  \texttt{Alg}
 \begin{eqnarray}\Reg(T,\cV,\texttt{Alg}) 
 & = &  \sum_{a^\star\in\cL} \cc_{a^\star}(\rho_{a^\star}-\rho^{\star} )\left(T- \E_\cV[N_{a^\star}(T)]\right) + \sum_{a\in\cN} \cc_{a} [\rho^{\star}-\rho_a] \E_\cV[N_a(T)] \nonumber \\& & \ \ + (\rho^\star - \rho)\left(BT - \sum_{a=1}^K c_a\bE_\cV[N_a(T)]\right). \label{eq:regretDecomp2} \end{eqnarray}
\end{proposition}

This decomposition writes the regret as  a sum of three non-negative terms. In
order for the regret to be small, each optimal arm $a^\star \in \cL$ should be
drawn very often 
(of order  $T$ times, to  make the first term  small) and each  suboptimal arm
$a^\star \in \cN$ should be drawn seldomly 
(to make  the second term  small).  Finally  if $\rho^\star>\rho$, that  is if
there are sufficiently many 
`worthwhile' arms to exceed the  budget, then the third term appears
as a  penalty for not using  the whole budget  at every round.  It  means that
arms on the margin $\cM$ have to be drawn sufficiently
  often so as to saturate the budget constraint.

\paragraph{An extended bandit interpretation.} Here we propose another view on
this regret decomposition,  by means of an extended bandit  game with an extra
arm, which we term a pseudo-arm, that  represents the choice not to pull arms.
Whenever an algorithm does not  saturate the budget constraint \eqref{SoftBC},
one can view this  algorithm as putting weight on a  pseudo-arm in the bandit,
that  yields  zero  gain  but  permits  saturation  of  the  budget.   Letting
$\mu_{K+1}=B\rho$  and  $c_{K+1}=B$,  the  gain associated  with  drawing  arm
$(K+1)$ (whose  distribution is a  point mass at  $B\rho$) is indeed  zero (as
$\mu_{K+1}  -  \rho  c_{K+1}  =  0$)   and,  for  any  $\bm  q(t)$  such  that
$\sum_{a=1}^{K}q_a(t)  c_a  \leq  B$,  there  exists  $q_{K+1}(t)$  such  that
$\sum_{a=1}^{K+1}q_a(t)  c_a =  B$,  as $c_{K+1}=B$.   Any  algorithm for  the
original bandit  problem selecting  $\hat{S}(t) \in \cS_{K}$  at time  $t$ can
thus be  viewed as an  algorithm selecting $\tilde{S}(t) \in  \cS_{K+1}$, that
additionally  includes  arm $(K+1)$  with  probability  $q_{K+1}(t)$.  As  the
pseudo-arm  is associated  with a  null gain,  the cumulated  gain and  regret
are     similar     in      both     settings.      Moreover,     as
$q_{K+1}(t) = (B - \sum_{a=1}^Kc_aq_a(t))/B$,  one easily sees that the number
of (artificial) selections of the pseudo-arm is such that
\[B\bE[N_{K+1}(T)]    =   BT    -    \sum_{a=1}^K   c_a\bE[N_a(T)],\]    which
equals the  third term in  the regret decomposition,  up to
  the factor $(\rho^{\star} - \rho)$.

In  this  extended  bandit  model,  the  three  sets  of  arms  introduced  in
Proposition~\ref{prop:Oracle}          remain         unchanged,          with
$\cL\equiv      \{a      \in     \{1,\dots,K+1\}:      \rho_a>\rho^{\star}\}$,
$\cM\equiv    \{a    \in     \{1,\dots,K+1\}:    \rho_a=\rho^{\star}\}$    and
$\cN\equiv   \{a    \in   \{1,\dots,K+1\}   :    \rho_a<\rho^{\star}\}$.    As
$\rho_{K+1}=\rho \leq \rho^\star$, the pseudo-arm  may only belong to $\cM$ or
$\cN$, and the  margin $\cM$ is always non-empty. Considering  the
extended bandit  model, the regret  decomposition can  be rewritten in  a more
compact way:
\[\Reg(T,\cV,           \texttt{Alg})          =           \sum_{a^\star\in\cL}
  \cc_{a^\star}(\rho_{a^\star}-\rho^{\star}                          )\left(T-
    \E[N_{a^\star}(T)]\right)  + \sum_{a\in\cN}  \cc_{a} [\rho^{\star}-\rho_a]
  \E[N_a(T)].\]

Our proofs make use  of this extended bandit model, since  many of the results
we present apply to both the ``actual'' arms $a=1,\ldots,K$ and the pseudo-arm
$(K+1)$. Our  proofs also  make use  of a  set $\cS$,  which, in  the extended
bandit model,  refers to all arms  in $(\cL\cup\cM)\backslash\{K+1\}$ whereas,
in the unextended bandit model, it refers simply to all optimal arms
both on and away from the margin.

\subsection{Related work}

There has been considerable work on various forms of ``budgeted'' or ``knapsack'' bandit problems \citep{TranThanhetal2012,Badanidiyuruetal2013,Agrawal&Devanur2014,Xiaetal2015,Xiaetal2016,Li&Xia2017}. The main difference between our work and these works is that we consider a round-wise budget constrain, and allow for several arms to be selected at each round, possibly in a randomized way in order to satisfy the budget constraint in expectation. In contrast, in most existing works, one arm is (deterministically) selected at each round, and the game ends when a global budget is exhausted. The work of \cite{Xiaetal2016b} appears to be the most closely related to ours: in their setup the agent may play multiple arms at each round, though the number of arms pulled at each round is fixed and the cost of pulling each arm is random and observed upon pulling each arm. \cite{Sankararaman18} also consider a framework in which a subset of arms is selected at each round, but this subset is chosen from a list of candidate subsets (as in a combinatorial bandit problem) and there is a global budget constraint. Compared to all these mentioned budgeted bandit problems, the focus of our analysis differs substantially, in that our primary objective is to not only prove rate optimality, but also leading constant optimality of our regret bounds. Proving constant optimality is especially challenging in situations where the set of optimal arms is non-unique, but we give careful arguments that overcome this challenge.

Several other extensions  of the multiple-play bandit model  have been studied
in the literature.  UCB algorithms have  been widely used in the combinatorial
semi-bandit setting,  in which at each  time step a  subset of arms has  to be
select among  a given class  of subsets, and  the rewards of  every individual
arms  in the  subset are  observed.   The most  natural  use of  UCBs and  the
``optimism in face of uncertainty principle''  is to choose at every time step
the subset  that would  be the  best if the  unknown means  were equal  to the
corresponding         UCBs.         This         was        studied         by
\cite{Chenetal2013,Kvetonetal2014,Kvetonetal2015c}, who exhibit good empirical
performance  and logarithmic  regret  bounds.  \cite{Combesetal2015b}  further
study  instance-dependent  optimality  for  combinatorial  semi  bandits,  and
propose an algorithm  based on confidence bounds on the  value of each subset,
rather than on confidence bounds on  the arms' means.  Their ESCB algorithm is
proved to be order-optimal for several combinatorial problems. As a by product
of our results, we will see  that in the multiple-play setting, using KL-based
confidence  bounds on  the arms'  means  is sufficient  to achieve  asymptotic
optimality. Another interesting  direction of extension is  the possibility to
have only partial feedback over the  $m$ proposed item. Variants of KL-UCB and
Thompson   Sampling   were   proposed   for   the   Cascading   bandit   model
\citep{Kvetonetal2015,Kvetonetal2015b},        Learning         to        Rank
\citep{Combesetal2015} or the  Position-Based model \citep{Lagreeetal2016}. It
would be interesting  to try to extend  the results presented in  this work to
these partial feedback settings.

\section{Regret Lower Bound} \label{sec:lb}

We first give  in Lemma~\ref{lem:lb} asymptotic lower bounds on  the number of
draws of suboptimal arms, either in high-probability or in expectation, in the
spirit              of              those             obtained              by
\cite{Lai&Robbins1985,Anantharametal1987}. Compared to  these works, the lower
bounds  obtained here  hold  under our  more general  assumptions  on the  arm
distributions,     which     is     reminiscent     of     the     work     of
\cite{Burnetas&Katehakis1996}.

To be  able to state  our regret lower bound,  we now introduce  the following
notation.    We  let   $\KL(\nu,\nu')$   denote   the  KL-divergence   between
distributions $\nu$ and $\nu'$. If $\nu$ and $\nu'$ are uniquely parameterized
by  their  respective  means  $\mu$  and  $\mu'$  as  in  a canonical single  parameter
exponential family (e.g. Bernoulli distributions),  then we abuse notation and
let     $\KL(\mu,\mu')\equiv    \KL(\nu,\nu')$.      For    a     distribution
$\nu\in\cD$ and a real $\mu$, we define
\begin{align}
\Kinf(\nu,\mu)&\equiv \inf\left\{\KL(\nu,\nu') : \nu'\in\cD\textnormal{ and }\mu<E(\nu')\textnormal{ and }\nu\ll\nu'\right\}, \label{eq:Kinfdef}
\end{align}
with the convention that $\Kinf(\nu,\mu)=\infty$ if there does not exist a $\nu\ll\nu'$ with $\mu<E(\nu')$. We will also use the convention that, for finite constants $d_1,d_2$, $d_1/(d_2 + \Kinf(\nu,\mu))=0$ when $\Kinf(\nu,\mu)=\infty$. We make one final assumption, and introduce two disjoint sets $\underline{\cN}$ and $\overline{\cN}$, whose union is $\cN$. The assumption is that, for each arm $a\in\{1,\ldots,K\}$, $\mu_a$ falls below the upper bound of the expected reward parameter space, i.e. $\mu_a < \mu_{+}\equiv \sup\{E(\nu) : \nu\in\cD\}$. We define the sets $\underline{\cN}$ and $\overline{\cN}$ respectively as the subsets of $\cN$ for which optimality is and is not feasible given our parameter space, namely
\begin{align*}
&\underline{\cN}\equiv \left[\cN\cap\left\{a : \cc_a \rho^\star< \mu_{+}\right\}\right]\backslash\{K+1\} \\
&\overline{\cN}\equiv \left[\cN\cap\left\{a : \cc_a \rho^\star\ge \mu_{+}\right\}\right]\backslash\{K+1\}.
\end{align*}
By defining  $\underline{\cN}$ and  $\overline{\cN}$ in  this way,  these sets
agree in the extended and unextended bandit models. The lower bounds presented
in this section will also agree in these two models.

We now define  a uniformly efficient algorithm, that generalizes  the class of
algorithms    considered    in     \cite{Lai&Robbins1985}.     An    algorithm
\texttt{Alg}     is    uniformly     efficient    if,     for    all
$\cV  \in \cD_K$ and $\alpha\in(0,1)$,  $\Reg(T,\cV,\texttt{Alg}) =  o(T^\alpha)$
as $T$ goes to  infinity (from now on, the limits in  $T$ will be for
  $T \to \infty$).  From the  regret
decomposition \eqref{eq:regretDecomp2}, this is equivalent to
\begin{enumerate}
  \item $T-\E_{{\mathcal{V}}}\left[N_{a^\star}(T)\right] = o\left(T^\alpha\right)$ for all arms $a^\star$ such that $\rho_{a^\star}>\rho^\star(\bm\mu)$;
  \item \label{it:cond2} $\E_{\mathcal{V}}[N_a(T)]=o(T^{\alpha})$ for all arms $a$ such that $\rho_a < \rho^\star(\bm\mu)$;
  \item if $\rho^\star(\bm\mu) > \rho,$ $BT - \sum_{a=1}^K c_a\bE_{\cV}[N_a(T)] = o(T^\alpha)$,
\end{enumerate}
where above and throughout we write $\E_{\cV}$ when we wish to emphasize that the expectation is over $\cV$.

\begin{lemma}[Lower bound on suboptimal arm draws] \label{lem:lb}
If an algorithm is uniformly efficient, then, for any arm $a\in(\cM\cup\underline{\cN})\backslash\{K+1\}$ and any $\delta\in(0,1)$ and $\epsilon>0$,
\begin{align}
&\lim_T \Prob\left\{N_a(T)< (1-\delta)\frac{\log T}{\Kinf(\nu_a,\cc_a \rho^{\star}) + \epsilon}\right\}=0. \label{eq:problb} \\
\intertext{One can take $\epsilon=0$ if $a\in\underline{\cN}$. Furthermore, for any suboptimal arm $a\in\underline{\cN}$,}
&\liminf_T \frac{\E[N_a(T)]}{\log T}\ge \frac{1}{\Kinf(\nu_a,\cc_a \rho^{\star})}. \label{eq:explb}
\end{align}
\end{lemma}
We defer the proof of this result to Appendix~\ref{app:lbproof}. We note that, while \eqref{eq:explb}  could also  easily be
obtained  using   the  recent   change-of-distribution  tools   introduced  by
\cite{Garivieretal2016}, we  need to  go back  to Lai  and Robbins'
technique to  prove the high-probability result  \eqref{eq:problb}, which will
be crucial in the  sequel.  Indeed, we will use it to  prove optimal regret of
our algorithms: in  essence we need to ensure that  we have enough information
about arms in  $\cM\cup\cN$ to ensure that  we pull the optimal  arms in $\cL$
sufficiently often.

We now present a corollary to Lemma~\ref{lem:lb} which provides a regret lower bound, as well as sufficient conditions for an algorithm to asymptotically match it. As already noted by \cite{Komiyamaetal2015} in the Bernoulli case for the bandit with multiple-play problems, an algorithm achieving the asymptotic lower bound \eqref{eq:explb} on the expected number of draws of arms in $\underline{\cN}$ does not necessarily achieve optimal regret, unlike in classic bandit problems. Thus, we emphasize that the upcoming condition \eqref{eq:subopt} alone is not sufficient to prove asymptotic optimality. The conditions of this proof can be easily obtained from the regret decomposition \eqref{eq:regretDecomp2}, and so the proof is omitted.

\begin{theorem}[Regret lower bound] \label{thm:reglb}
If an algorithm \texttt{Alg} is uniformly efficient, then
\begin{align}
\liminf_T \frac{\Reg(T,\cV,\texttt{Alg})}{\log T}&\ge \sum_{a\in\underline{\cN}} \frac{\cc_a(\rho^{\star}-\rho_a)}{\Kinf(\nu_a,\cc_a\rho^{\star})}. \label{eq:reglb}
\end{align}
Moreover, any algorithm \texttt{Alg} satisfying
\begin{align}
\mbox{for arms $a\in \underline{\cN}$: }\ \ &\E_\cV[N_a(T)]= \frac{\log T}{\Kinf(\nu_a,\cc_a \rho^{\star})} + o(\log T), \label{eq:subopt} \\
\mbox{for arms $a\in \overline{\cN}$}\ \ &\E_\cV[N_a(T)]= o(\log T), \label{eq:suboptindifference} \\
\mbox{for arms $a^\star\in \cL$: }\ \ &\E_\cV[N_{a^\star}(t)]= T-o(\log T), \label{eq:nonmargin} 
\end{align}
and, if $\rho^\star(\bm\mu) > \rho$, 
\begin{equation}
BT - \sum_{a=1}^K c_a\bE_\cV[N_a(T)] = o(\log(T)), \label{eq:exhaustbudget}\end{equation}
is asymptotically optimal, in the sense that it satisfies 
\begin{align}
\limsup_T \frac{\Reg(T,\cV,\texttt{Alg})}{\log T}\leq \sum_{a\in\underline{\cN}} \frac{\cc_a(\rho^{\star}-\rho_a)}{\Kinf(\nu_a,\cc_a\rho^{\star})}. \label{eq:regub}
\end{align}
\end{theorem}

\section{Algorithms} \label{sec:algs}

Algorithms rely on estimates of the arm distributions and their means, that we formally introduce below. For each arm $a$ and natural number $n$, define $\tau_{a,n}=\min\{t\ge 1 : N_a(t)=n\}$ to be the (stopping) time at which the $n^{\textnormal{th}}$ draw of arm $a$ occurs.
Let $X_{a,n}\equiv Y_a(\tau_{a,n})$ denote the $n^{\textnormal{th}}$ draw from $\nu_a$. One can show that $\{X_{a,n}\}_{n=1}^\infty$ is an i.i.d. sequence of draws from $\nu_a$ for each $a$, though we note that our variation independence assumption is too weak to ensure that these sequences are independent for two arms $a\not=a'$ (this is not problematic -- most of our arguments end up focusing on arm-specific sequences $\{X_{a,n}\}_{n=1}^\infty$)\footnote{It is \textit{a priori} possible that $\tau_{a,n}=\infty$ for all $n$ large enough (though, as we showed in Section~\ref{sec:lb}, this event will occur with probability zero for any reasonable algorithm). To deal with this case, let $X_{a,n}\equiv Y_a(\tau_{a,n})$ denote the $n^{\textnormal{th}}$ draws from $\nu_a$ for all $\tau_{a,n}<\infty$ and let $\{X_{a,n}\}_{n : \tau_{a,n}=\infty}$ denote an i.i.d. sequence independent of $\{X_{a,n}\}_{n : \tau_{a,n}<\infty}$.}. We denote the empirical distribution function of observations drawn from arm $a$ by any time $T$ by
\begin{align*}
\hat{\nu}_a(T)&\equiv \frac{1}{N_a(T)}\sum_{t=1}^T \delta_{Y_a(t)} \Ind\{a\in\A(t)\} = \frac{1}{N_a(T)}\sum_{n=1}^{N_a(T)} \delta_{X_{a,n}}.
\end{align*}
We  similarly  define  $\hat{\nu}_{a,n}$  to  be  the  empirical  distribution
function of the observations
$X_{a,1}$, $\ldots$,  $X_{a,n}$.  Thus, $\hat{\nu}_a(t)=\hat{\nu}_{a,N_a(t)}$.
We further  define $\hat{\mu}_a(t)$ to  be the empirical mean  of observations
drawn from arm $a$ by time $t$ and $\hat{\mu}_{a,N_a(t)}=\hat{\mu}_a(t)$.

\subsection{KL-UCB} \label{sec:kl}
\afterpage{\begin{algorithm}[ht] \label{alg:kl}
\caption{KL-UCB}

% Adding footnotes inside a float (the algorithm environment) messes up the counting
% Save the footnote counter
\newcounter{savefootnote}
\setcounter{savefootnote}{\value{footnote}}

\begin{algorithmic}[0]
  \State       \textit{Parameters}       A       non-decreasing       function
  $f  : \mathbb{N}\rightarrow\mathbb{R}$  and an  operator $\Pi_{\cD}$
  mapping each empirical  distribution functions $\hat{\nu}_a(t)$ to
  an element of the  model $\cD$.  \State \textit{Initialization} Pull
  each      arm       of      $\{1,\ldots,K\}$      once.\protect\footnotemark
  \For{$t=K,K+1,\ldots,  T-1$}  \State  For $a=1,\ldots,K$,  let  $U_a(t)$  be
  defined        as       in        \eqref{eq:Uadef}.        \State        Let
  $\hat{\rho}^\star(t)\equiv  \rho^\star\left(U_a(t)  :  a=1,\ldots,K\right)$.
  \State              For               $a\in\{1,\ldots,K\}$,              let
  $q_a(t)         =         \Ind\{U_a(t)>\cc_a         \hat{\rho}^\star(t)\}$.
  \If{$\widehat{\cM}(t)\equiv\{a  :   U_a(t)=  \cc_a\hat{\rho}^\star(t)\}$  is
    non-empty}        \If{$\hat{\rho}^\star(t)>\rho$}        \State        For
      $a\in\widehat{\cM}(t)$,    let  $q_a(t)=\big[B-\sum_{a      :
        U_a(t)>\cc_a\hat{\rho}^\star(t)}   \cc_a\big]/\sum_{a\in\widehat{\cM}(t)}\cc_a$.    \Else   \hspace{0.5em}   Let
  $q_a(t)=0$ for all $a\in\widehat{\cM}(t)$\footnotemark.
    \EndIf
  \EndIf
  \State Draw $\A(t+1)$ from any distribution $Q(t)$ with marginal probabilities $q_a(t)$\footnotemark.
  \State Draw arms in $\A(t+1)$ and observe $Y_a(t+1)$, $a\in\A(t+1)$.
\EndFor
\end{algorithmic} % On the next line, I add one and then two to 'savefootnote' and print that number. No need to reset footnote counter - it's correctly keeping track
\end{algorithm} \footnotetext[\the\numexpr\value{savefootnote}+1\relax]{If $\cc_a\le B$ for all $a$, then it is always possible to do this with $K$ draws and respect the budget. In particular, at time $t=a$, draw arm $a$ with probability one. If, for some $a$, $\cc_a> B$, then this strategy will violate the budget constraint, though a stochastic strategy that draws these arms $a$ with probability $\cc_a/B$ until the (random) stopping time at which the first draw occurs would respect the budget. Whether we use this strategy or just pull each arm once has essentially no effect on our analysis, and so for simplicity we assume the agent draws each arm once to initialize the algorithm.}
\footnotetext[\the\numexpr\value{savefootnote}+2\relax]{\label{foot:extraif}While presumably not necessary, this restriction aids our arguments in Section~\ref{sec:suboptrare}, and seems very mild given that at $\rho$ the agent is indifferent to whether she pulls an arm or a pseudo-arm.}\footnotetext[\the\numexpr\value{savefootnote}+3\relax]{An easy choice is to make $Q(t)$ a product measure with marginal probabilities $q_1(t),\ldots,q_{K+1}(t)$, but  this choice is not necessary, and more careful choices may reduce the probability of overspending the budget at any given time point.}}

At time $t$, UCB algorithms leverage  high probability upper bound $U_a(t)$ on
$\mu_a$ for each $a$. The methods  used to build these confidence bounds vary,
as does the way the algorithm uses these confidence bounds. In our setting, we
derive   these  bounds   using   the   same  technique   as   for  KL-UCB   in
\cite{Cappeetal2013}.   At the beginning of round $(t+1)$,  the  KL-UCB  algorithm computes  an
optimistic  oracle  strategy  $(q_a(t))_{a=1,\dots,K}$,   that  is  an  oracle
strategy  assuming the  unknown mean  of each  arm $a$  is equal  to its  best
possible value, $U_a(t)$.  From Proposition~\ref{prop:Oracle}, this optimistic
oracle                                depends                               on
$\hat{\rho}^\star(t) = \rho^\star\left(U_a(t) : a=1,\dots,K\right)$,
where     $\rho^\star(\bm\mu)$     is      the     function     defined     in
Proposition~\ref{prop:Oracle}.  Then each arm  is included in $\hat{\cA}(t+1)$
independently with  probability $q_a(t)$.  Due  to the structure of  an oracle
strategy,  KL-UCB  can  be  rephrased  as successively  drawing  the  arms  by
decreasing order of the ratio $U_a(t)/c_a$  until the point that the budget is
exhausted, with some probability to include the arms on the margin.  We choose
to  keep  the name  KL-UCB  for  this  straightforward generalization  of  the
original KL-UCB algorithm.

The  definition of  the upper  bound $U_a(t)$  is closely  related to  that of
$\Kinf$   given  in   \eqref{eq:Kinfdef}.    Let   $\Pi_{\cD}$  be   a
problem-specific    operator   mapping    each   empirical
distribution   function  $\hat{\nu}_a(t)$   to   an  element   of  the   model
$\cD$.    Furthermore,    let
$f : \mathbb{N}\rightarrow\mathbb{R}$ be a non-decreasing function, where this
function is  usually chosen  so that  $f(t)\approx \log t$.   The UCB  is then
defined as
\begin{align}
U_a(t)&\equiv \sup\left\{E(\nu) :  \nu\in\cD\textnormal{ and }\KL\left(\Pi_{\cD}\left(\hat{\nu}_a(t)\right),\nu\right)\le \frac{f(t)}{N_a(t)}\right\}\textnormal{, $a=1,\ldots,K$}. \label{eq:Uadef}
\end{align}
As we will see, the closed form expression for $U_a(t)$ can be made slightly more explicit for exponential family models, though the expression still has the same general flavor. If a number $\mu$ satisfies $\mu\ge U_a(t)$, then this implies that, for every $\nu\in\cD$ for which $E(\nu)>\mu$, $\KL\left(\Pi_{\cD}(\hat{\nu}_a(t)),\nu\right)>\frac{f(t)}{N_a(t)}$. Consequently, $\Kinf(\Pi_{\cD}(\hat{\nu}_a(t)),\mu)\ge \frac{f(t)}{N_a(t)}$.

We now describe two settings in which the algorithm that we have described achieves the optimal asymptotic regret bound. These two settings and the presentation thereof follows \cite{Cappeetal2013}. The first family of distributions we consider for $\cD$ is a canonical one-dimensional exponential family $\mathcal{E}$. For some dominating measure $\lambda$ (not necessarily Lebesgue), open set $H\subseteq\mathbb{R}$, and twice-differentiable strictly convex function $b : H\rightarrow\mathbb{R}$, $\mathcal{E}$ is a set of distributions $\nu_\eta$ such that
\begin{align*}
\frac{d\nu_\eta}{d\lambda}(x)&= \exp\left[x\eta-b(\eta)\right].
\end{align*}
We assume that the  open set $H$ is the natural parameter  space, i.e. the set
of all $\eta\in\mathbb{R}$ such that $\int \exp(x\eta) d\lambda(x)<\infty$. We
define    the     corresponding    (open)     set    of     expectations    by
$I\equiv \{E(\nu_\eta) : \eta\in  H\}\equiv (\mu_{-},\mu_{+})$ and its closure
by   $\bar{I}=[\mu_{-},\mu_{+}]$.   We   have  omitted   the   dependence   of
$\mathcal{E}$ on $\lambda$ and $b$ in the notation. It is easily verified that
$\Kinf(\mu_a,\cc_a \rho^\star)= \KL(\nu_a,\cc_a \rho^\star)$.

For    the   moment    suppose    that   $\hat{\nu}_a(t)$    is   such    that
$\hat{\mu}_a(t)\in  I$. In  this case  we let  $\Pi_{\cD}$ denote  the
maximum            likelihood            operator           so            that
$\Pi_{\cD}\left(\hat{\nu}_a(t)\right)$ returns the unique distribution
in      $\cD$      indexed      by     the      $\eta$      satisfying
$b'(\eta)=\hat{\mu}_a(t)$. Thus, in this  setting where $\hat{\mu}_a(t)\in I$,
the UCB $U_a(t)$ then takes the form of the expression in \eqref{eq:Uadef}.

More generally, we must deal with the case that $\hat{\mu}_a(t)$ equals $\mu_{+}$ or $\mu_{-}$. For $\mu\in I$, define by convention $\KL(\mu_{-},\mu) = \lim_{\mu'\rightarrow\mu_{-}} \KL(\mu_{-},\mu)$, $\KL(\mu_{+},\mu) = \lim_{\mu'\rightarrow\mu_{+}} \KL(\mu',\mu)$, and analogously for $\KL(\mu,\mu_{-})$ and $\KL(\mu,\mu_{+})$. Finally, define $\KL(\mu_{-},\mu_{-})$ and $\KL(\mu_{+},\mu_{+})$ to be zero. This then gives the following general expression for $U_a(t)$ that we use to replace \eqref{eq:Uadef} in the KL-UCB Algorithm:
\begin{align}
U_a(t)&\equiv \sup\left\{\mu\in\bar{I} :  \KL\left(\hat{\mu}_a(t),\mu\right)\le \frac{f(t)}{N_a(t)}\right\}\textnormal{, $a=1,\ldots,K$}\label{indexKLUCB}.
\end{align}
Note that  this definition of $U_a(t)$  does not explicitly include  a mapping
$\Pi_{\cD}$ mapping  any empirical distribution  function to
an element  of the model $\cD$.  Thus
we have avoided any problems that could  arise in defining such a mapping when
$\hat{\mu}_a(t)$ falls on the boundary of~$\bar{I}$. The above optimization problem can be solved by noting that $\mu\mapsto \KL\left(\hat{\mu}_a(t),\mu\right)$ is convex, and so one can first identify the $\mu_0$ minimizing this function, and then perform a root-finding method for monotone functions to (approximately) identify the largest $\mu\ge \mu_0$ at which $\KL\left(\hat{\mu}_a(t),\mu\right)- \frac{f(t)}{N_a(t)}=0$.

The KL-UCB variant that we have presented achieves the asymptotic regret bound in the setting where $\cD=\mathcal{E}$.
\begin{theorem}[Optimality      for      single     parameter      exponential
  families]  \label{thm:expfam}  Suppose  that  $\cD  =  \mathcal{E}$.
  Further let $f(t)=\log  t + 3\log\log t$ for $t\ge  3$ and $f(1)=f(2)=f(3)$.
  This      variant      of      KL-UCB      satisfies      \eqref{eq:subopt},
  \eqref{eq:suboptindifference},            \eqref{eq:nonmargin}           and
  \eqref{eq:exhaustbudget}. Thus, KL-UCB achieves  the asymptotic regret lower
  bound \eqref{eq:reglb} for uniformly efficient algorithms.
\end{theorem}
Another interesting family of distributions for $\cD$ is a set $\mathcal{B}$ of distributions on $[0,1]$ with finite support. If the support of $\cD$ is instead bounded in some $[-M,M]$, then the observations can be rescaled to $[0,1]$ when selecting which arm to pull using the linear transformation $x\mapsto (x + M)/(2M)$.

If $\cD$ is equal to $\mathcal{B}$, then \cite{Cappeetal2013} observe that \eqref{eq:Uadef} rewrites as
\begin{align*}
U_a(t)&= \sup\left\{E(\nu) : \Supp[\nu]\subseteq \Supp\left[\hat{\nu}_a(t)\right]\cup\{1\}\textnormal{ and }\KL\left(\hat{\nu}_a(t),\nu\right)\le\frac{f(t)}{N_a(t)}\right\}
\end{align*}
where, for a measure $\nu'$, we use $\Supp[\nu']$ to denote the support of $\nu'$. They furthermore observe that this expression admits an explicit solution via the method of Lagrange multipliers.

\begin{theorem}[Optimality for finitely supported distributions] \label{thm:finsup}
Suppose that $\cD=\mathcal{B}$. Let $\Pi_{\cD}$ denote the identity map and $f(t)=\log t + \log\log t$ for $t\ge 2$ and $f(1)=f(2)$. Suppose that $\mu_a\in(0,1)$ for all $a=1,\ldots,K$. The variant of KL-UCB satisfies \eqref{eq:subopt}, \eqref{eq:suboptindifference}, \eqref{eq:nonmargin} and \eqref{eq:exhaustbudget}. Thus, KL-UCB achieves the asymptotic regret lower bound \eqref{eq:reglb} for uniformly efficient algorithms.
\end{theorem}
In  both theorems,  the  little-oh notation  hides  the problem-dependent  but
$T$-independent  quantities. In  the proofs  of Theorems  \ref{thm:expfam} and
\ref{thm:finsup}  we refer  to  equations in  \cite{Cappeetal2013b} where  the
reader can find explicit  finite-sample, problem-dependent expressions for the
$o(\log  T)$   term  in  \eqref{eq:subopt}   for  the  settings   of  Theorems
\ref{thm:expfam}  and  \ref{thm:finsup}.   The   argument  used  to  establish
\eqref{eq:suboptindifference}  considers similar  $o(\log T)$  terms to  those
that appear  in the proof  of \eqref{eq:subopt}, though the  simplest argument
for establishing \eqref{eq:suboptindifference} (which, for brevity, is the one
that we have elected to present  here) invokes asymptotics.  The argument used
to establish \eqref{eq:nonmargin} in these  settings, on the other hand, seems
to be  fundamentally asymptotic  and does  not appear  to easily  yield finite
sample constants. Nonetheless, this is to  our knowledge the first handling of
thick  margins  in  the multiple-play  bandit
literature, and  so we believe that  our rate- and constant-optimal regret guarantee is of interest  despite its asymptotic
nature.

Moreover, though not presented in detail here, our proof techniques can be used to establish a finite-time regret guarantee that is rate-optimal, namely is $O(\log T)$, but is constant-suboptimal. To obtain this bound, we note that, by Proposition~\ref{prop:RegretDec}, it suffices to combine (i) the previously-discussed finite-time variants of \eqref{eq:subopt} and \eqref{eq:suboptindifference} that can result from the proof of Theorem~\ref{thm:finsup} and (ii) the following finite-time variant of \eqref{eq:nonmargin}, which must hold for all $T\ge 1$ and some $C>0$:
\begin{align}
\mbox{for arms $a^\star\in \cL$: }\ \ &\E_\cV[N_{a^\star}(t)]= T-C \log T. \label{eq:finitetimenonmargin}
\end{align}
This guarantee is asymptotically weaker than that in \eqref{eq:nonmargin} in the sense that the $o(\log T)$ term has been replaced by $O(\log T)$, but is stronger than \eqref{eq:nonmargin} in the sense that we require a finite-time bound on the $O(\log T)$ term rather than only an asymptotic guarantee. Though we did not explicitly establish the above in our proof of Theorem~\ref{thm:finsup}, only a minor modification to the proof is needed. Specifically, by \eqref{eq:keyobs}, it suffices to obtain a finite-time upper bound on $\E[M_a^{a^\star}(T)]$ for all $a\in\cM\cup\cN$ and $a^\star\in\cL$. This upper bound can be found by noting that the proof of Lemma~\ref{lem:TermA} shows that $\E[M_a^{a^\star}(T)]\le O(\log T)$, and explicit finite-sample constants can be computed for this bound just as they can for \eqref{eq:subopt}. Plugging this into \eqref{eq:keyobs} then establishes \eqref{eq:finitetimenonmargin}, which in turn establishes a finite-time $O(\log T)$ regret bound. This finite-time regret bound will be valid even if $\cM$ contains more than one arm.

\subsection{Thompson Sampling} \label{sec:thom}
\begin{algorithm}[ht] \label{alg:tom}
\caption{Thompson Sampling}
\begin{algorithmic}[0]
\State \textit{Parameters} For each arm $a=1,\ldots,K$, let $\Pi_a(0)$ be a prior distribution on $\mu_a$.
\For{$t=0,1,\ldots$}
  \State For each arm $a=1,\ldots,K$, draw $\theta_a(t)\sim \Pi_a(t)$.
  \State Let $\hat{\rho}^\star(t)\equiv \rho^\star\left(\left(\theta_a(t) : a=1,\ldots,K\right)\right)$.
   \State For $a\in\{1,\ldots,K\}$, let $q_a(t) = \Ind\{\theta_a(t)>\cc_a \hat{\rho}^\star(t)\}$.
  \If{$\widehat{\cM}(t)\equiv\{a : \theta_a(t)= \cc_a\hat{\rho}^\star(t)\}$ is non-empty}
      \State For $a\in\widehat{\cM}(t)$, let $q_a(t)=\big[B-\sum_{a : \theta_a(t)>\cc_a\hat{\rho}^\star(t)} \cc_a\big]/\sum_{a\in\widehat{\cM}(t)}\cc_a$.
  \EndIf
  \State Draw $\A(t+1)$ from any distribution $Q(t)$ with marginal probabilities $q_a(t)$.
  \State Draw the corresponding rewards $Y_a(t+1)$, $a\in\A(t+1)$.
  \State For each $a\in\A(t+1)$, obtain a new posterior $\Pi_a(t+1)$ by updating $\Pi_a(t)$ with the observation $Y_a(t+1)$.
  \State For each $a\not\in\A(t+1)$, let $\Pi_a(t+1)=\Pi_a(t)$.
\EndFor
\end{algorithmic}
\end{algorithm}
Thompson sampling  uses Bayesian ideas to  account for the uncertainty  in the
estimated  reward distributions.   In a  classical bandit  setting, one  first
posits  a (typically  non-informative)  prior  over the  means  of the  reward
distributions, and then at each time  updates the posterior and takes a random
draw of  the $K$ means  from the posterior and  pulls the arm  whose posterior
draw is the largest. In our setting, this corresponds to drawing the subset of
arms for  which the  posterior draw  to cost  ratio is  largest (up  until the
budget constraint  is met), which  generalizes the idea initially  proposed by
\cite{Thompson1933}. In the above algorithm,  we focus on independent priors so
that  the  only posteriors  updated  at  time $(t+1)$  are  those  of arms  in
$\A(t+1)$. %\change{Moreover, in this algorithm we focus on the case where the distribution is fully indexed by its mean.} 
At  time  $(t+1)$,  Thompson  Sampling  first  draws  one  sample
$\theta_a(t)$ from the posterior distribution on the mean of each arm $a$, and
then   selects    a   subset    according   an   oracle    strategy   assuming
$(\theta_a(t))_{a=1,\dots,K}$ are the true parameters. 

We prove  the optimality of Thompson  sampling for Bernoulli rewards,  for the
particular choice  of a uniform  prior distribution on  the mean of  each arm.
Note that the algorithm is easy to implement in that case, since $\Pi_a(t)$ is
a   Beta   distribution   with  parameters   $N_a(t)\hat{\mu}_a(t)   +1$   and
$N_a(t)(1-\hat{\mu}_a(t))+1$. Our proof relies on the same techniques as those
used  to prove  the optimality  of Thompson  sampling in  the standard  bandit
setting  for  Bernoulli rewards  by  \cite{Agrawal&Goyal2012}.   We note  that
\cite{Komiyamaetal2015}  also   made  use  of   some  of  the   techniques  in
\cite{Agrawal&Goyal2012}  to prove  the  optimality of  Thompson sampling  for
Bernoulli rewards in the multiple-play bandit setting.

\begin{theorem}[Optimality  for  Bernoulli  rewards] \label{thm:thom}  If  the
  reward distributions are Ber\-noul\-li and  $\Pi_a(0)$ is a standard uniform
  distribution   for    each   $a$,    then   Thompson    sampling   satisfies
  \eqref{eq:subopt},  \eqref{eq:suboptindifference}, \eqref{eq:nonmargin}  and
  \eqref{eq:exhaustbudget}.  Thus,  Thompson sampling achieves  the asymptotic
  regret lower bound \eqref{eq:reglb} for uniformly efficient algorithms.
\end{theorem}
For any $\epsilon>0$ and $a\in\underline{\cN}$, the proof shows that Thompson sampling satisfies
\begin{align*}
\E[N_a(T)]\le (1+\epsilon)^2\frac{f(T)}{\KL(\mu_a,\cc_a \rho^\star)} + o(\log T).
\end{align*}
The proof gives an explicit bound on the $o(\log T)$ term that depends on both the problem and the choice of $\epsilon$. We conclude by noting that, similarly as for KL-UCB, our proof techniques can be easily adapted to give a rate-optimal but constant-suboptimal finite-time regret bound, where this bound will be valid even if $\cM$ contains more than one arm.

\section{Numerical Experiments} \label{sec:exp}

We now run  four simulations to evaluate our theoretical  results in practice,
all with Bernoulli reward distributions, a horizon of $T=100\,000$, and $K=5$.
The   simulation  settings   are  displayed   in  Table~\ref{tab:simsettings}.
Simulations 1-3 are run using $5\,000$ Monte Carlo repetitions, and Simulation
4 was run  using $50\,000$ repetitions to reduce Monte  Carlo uncertainty. %The
%\texttt{R} \citep{R2014} code for running  one repetition of our simulation is
%available in the Supplementary Materials.

\begin{table}[h]
\centering
\begin{tabular}{l|lllllll}
& $\mu$ & $\cc$ & $B$ & $\rho$ & $\cL$ & $\cM$ & $\overline{\cN}$ \\
\hline
Sim 1 & $(0.5,0.45,0.45,0.4,0.3)$ & $(1,1,1,1,1)$ & $2$ & $0$ & $\{1\}$ & $\{2,3\}$ & $\emptyset$ \\
Sim 2 & $(0.7,0.6,0.5,0.3,0.2)$ & $(1,1,1,1,1)$ & $3$ & $0$ & $\{1,2\}$ & $\{3\}$ & $\emptyset$ \\
Sim 3 & $(0.5,0.45,0.45,0.4,0.3)$ & $(0.8,1,1,0.8,0.6)$ & $2$ & $0.5$ & $\{1\}$ & $\{4,5,6\}$ & $\emptyset$ \\
Sim 4 & $(0.7,0.6,0.5,0.3,0.2)$ & $(1.5,1,1,1,2.5)$ & $3$ & $0.4$ & $\{2,3\}$ & $\{1\}$ & $\{5\}$ \\
\end{tabular}
\caption{Simulation settings  considered. Simulations 1 and  3 have non-unique
  margins so that $q_a$  must be less than one for at  least one arm $a\in\cM$
  for the budget  constraint to be satisfied. In Simulation  3, the pseudo-arm
  $(K+1)=6$ is in $\cM$, and in Simulation 4 arm $5$ is in $\overline{\cN}$.}
\label{tab:simsettings}
\end{table}

For $d \in \mathbb{R}$, we define the  KL-UCB $d$ algorithm as the instance of
KL-UCB using the  function $f(t)=\log t +  d\log\log t$. Note that  the use of
both  KL-UCB~3 and  KL-UCB~1 are  theoretically  justified by  the results  of
Theorems  \ref{thm:expfam} and  \ref{thm:finsup},  as Bernoulli  distributions
satisfy the conditions of both theorems.  In the settings of Simulations 1 and
2, which  represent multiple-play  bandit instances  as $B$  is an  integer in
$[1,K]$ and the cost of pulling each  arm is one, we compare Thompson sampling
and  KL-UCB  to the  ESCB  algorithm  of \cite{Combesetal2015b}.   As
  quickly  explained  earlier,  ESCB  is   a  generalization  of  the  KL-UCB
algorithm, designed  for the combinatorial semi-bandit  setting (that includes
multiple-play). This algorithm computes an  upper confidence bound for the sum
of the arm means for each  of the ${K\choose B}$ candidate sets $\mathcal{S}$,
defined by the optimal value to
\begin{align}
\sup_{(\mu_1,\ldots,\mu_K)\in[0,1]^K} \sum_{a\in\mathcal{S}} \mu_K\,\textnormal{ subject to }\,\sum_{a\in\mathcal{S}} N_a(t) \KL\left(\hat{\mu}_a(t),\mu_a\right)\le f(t) \label{indexESCB}
\end{align}
and draws the arms in the set $\mathcal{S}$ with the maximal index. Just like KL-UCB, ESCB uses confidence bounds whose level rely on a function $f$ such that $f(t)\approx \log t$. Because the optimization problem solved to compute the indices \eqref{indexKLUCB} and \eqref{indexESCB} are different, the $f$ functions used by KL-UCB and ESCB are not directly comparable. Nonetheless, a side-by-side comparison of the two algorithms seems to indicate that $f(t)=\log t + cB\log\log t$ for ESCB is comparable to $f(t)=\log t + c\log\log t$ for KL-UCB. \citeauthor{Combesetal2015b} prove an $O(\log T)$ regret bound (with a sub-optimal constant) for the version of ESCB corresponding to the constant $c=4$, that we refer to as ESCB~4$B$.

\begin{figure}[ht]
  \centering
  \includegraphics[width=\linewidth]{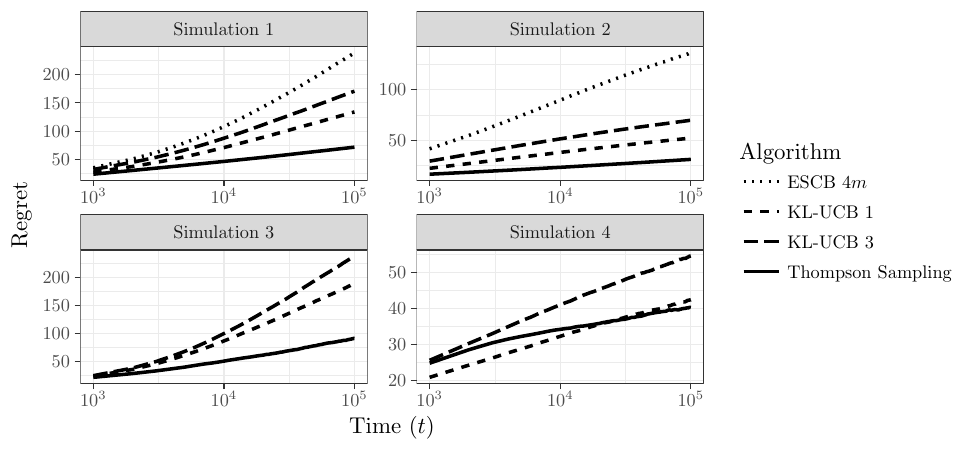}
  \caption{Regret of the four algorithms with theoretical guarantees. ESCB only run for Simulations 1 and 2 for which the cost is identically one for all arms.} 
  \label{fig:reg}
\end{figure}

Figure  \ref{fig:reg}  displays  the  regret   of  the  four  algorithms  with
theoretical  guarantees.    All  but   ESCB~4$B$  have   been  proven   to  be
asymptotically optimal,  and thus  are guaranteed  to achieve  the theoretical
lower bound asymptotically. In our finite sample simulation, Thompson sampling
performs  better than  this  theoretical guarantee  may suggest  (the
  regret lower  bounds at time  $T=100\,000$ are approximately equal  to $150$
  and $45$ in Simulations 1  and 2, respectively).  Indeed, Thompson sampling
outperforms  the KL-UCB  algorithms in  all but  Simulation~4, while  KL-UCB~1
outperforms KL-UCB~3 and  KL-UCB~3 outperforms ESCB~4$B$ in  Simulations 1 and
2.   To give  the  reader  intuition on  the  relative  performance of  KL-UCB
variants,  note   that  in  the   proofs  of  Theorems   \ref{thm:expfam}  and
\ref{thm:finsup} we prove that the number  of pulls on each suboptimal arm $a$
is upper bounded  by $f(T)/\Kinf(\nu_a,\cc_a\rho^\star) + o(\log  T)$, with an
explicit   finite  sample   constant  for   the  $o(\log   T)$  term.    While
$f(T) =  \log T +  o(\log T)$  for KL-UCB~1 and  KL-UCB~3, for finite  $T$ the
quantities $\log T$ and $\log T + c\log \log T$, $c=1,3$, are quite different.
At  $T=10^5$,  $\log T  +  \log\log  T$ is  20\%  larger  than $\log  T$,  and
$\log T + 3\log\log T$ is 60\%  larger. This difference does not decay quickly
with sample size: at $T=10^{15}$,  these two quantities are still respectively
10\% and 30\% larger than $\log T$.  This makes clear the practical benefit to
choosing $f(t)$  as close  to $\log  t$ as  is theoretically  justifiable: for
Bernoullis,  the  choice of  $f(t)$  in  Theorem~\ref{thm:finsup} yields  much
better results than the choice of $f(t)$ in Theorem~\ref{thm:expfam}.

We also compared the performance of KL-UCB~0 and ESCB~0 in Simulations 1 and~2
(details omitted here,  but the exact results of this  simulation are given in
Figure~2         of        the         earlier        technical         report
\citealp{Luedtke&Kaufmann&Chamnbaz2016}). Though  not theoretically justified,
this  choice of  $f(t)=\log t$  has been  used quite  a lot  in practice.  The
ordering of the three algorithms is the  same in Simulations 1 and 2: Thompson
Sampling performs best while ESCB~0 slightly outperforms KL-UCB~0.
This should however be mitigated by the gap of numerical complexity between the two algorithms, especially when $B$ and $K$ are large and $B/K$ is not close to $0$ or $1$: while KL-UCB only requires running $K$ univariate root-finding procedures regardless of $B$, the current proposed ESCB algorithm requires running ${K\choose B}$ univariate root-finding procedures. For $K=100$ and $B=10$, this is a difference of running $100$ root-finding procedures versus more than $10^{13}$ of them.

\begin{figure}[ht]
  \centering
  \includegraphics[width=0.9\linewidth]{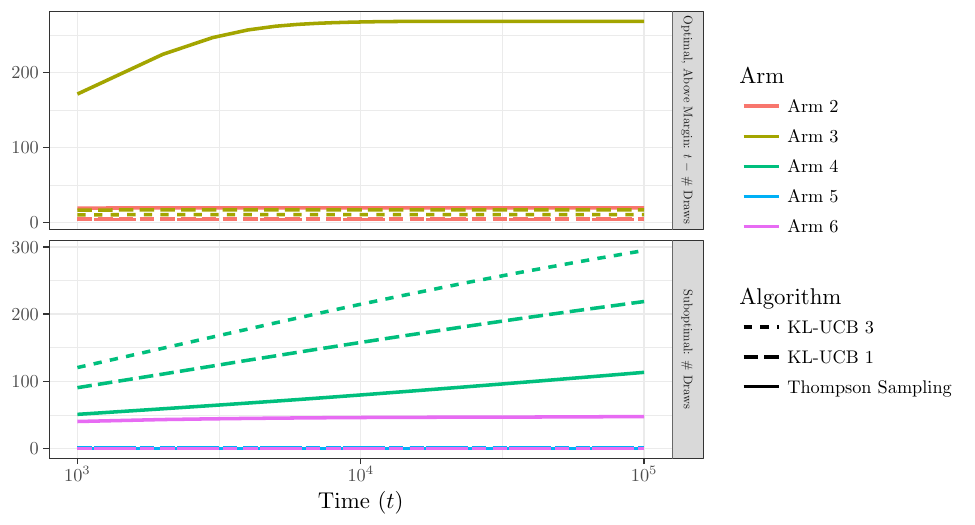}
  \caption{Time minus the number of optimal arm draws (top) and number of suboptimal arm draws (bottom) in Simulation 4.}
  \label{fig:N}
\end{figure}

Figure \ref{fig:N} displays the number of optimal and suboptimal arm draws in Simulation 4. None of the algorithms pulled the arm in $\overline{\cN}$ (arm 5) often. Thompson Sampling pulled the indifference point pseudo-arm surprisingly often in the first $10^3$ draws, and as a result arm 3 (above the margin) was also not pulled as often as would be expected in these early draws. By time $10^4$, the regret of Thompson sampling appears to have stabilized, and soon outperforms that of the two KL-UCB algorithms. We also checked what would happen if the indifference point were increased from $0.4$ to $0.45$ (details not shown). In this case, it takes even longer for the algorithm to differentiate between arm 3 (with $\rho_3=0.5$) and the pseudo-arm, though by time $10^5$ the algorithm again appears to have succeeded in learning that pulling arm 3 is to be preferred over pulling the peudo-arm.

\section{Proofs   of    Optimality   of    KL-UCB   and    Thompson   Sampling} \label{sec:proofoutlines}  We now outline our  proofs of optimality
for the KL-UCB and Thompson sampling schemes. We break this section into three
subsections. Section~\ref{sec:suboptrare} establishes that  the arms in $\cN$,
i.e.     the    suboptimal   arms,    are    not    pulled   often    (satisfy
Equations~\ref{eq:subopt}~and~\ref{eq:suboptindifference}).    Due    to   the
differences in  proof methods,  we consider the  KL-UCB and  Thompson sampling
schemes separately in  this subsection.  Section~\ref{sec:budgetsat} justifies
that when $\rho^\star>\rho$, the budget constraint  is most often saturated, that is
the    third     term    in    the    regret     is    negligible.     Finally
Section~\ref{sec:optcommon}  establishes that  the  arms in  $\cL$, i.e.   the
optimal   arms   away   from   the   margin,   are   pulled   often   (satisfy
Equation~\ref{eq:nonmargin}). We give the outline of the proofs for the KL-UCB
and Thompson sampling  schemes simultaneously, though we  provide the detailed
arguments                             separately                            in
Appendices~\ref{app:klucbproof}~and~\ref{app:thomproof}, respectively. We note
that  the order  of  presentation of  the two  subsections  is important:  the
arguments  used  in  Section~\ref{sec:optcommon}   rely  on  the  validity  of
\eqref{eq:subopt} and  \eqref{eq:suboptindifference}, which is  established in
Section~\ref{sec:suboptrare}.

To  ease the  presentation, we  find it  convenient to  consider the  extended
bandit  model   presented  in   Section~\ref{sec:regretdecomp},  in   which  a
pseudo-arm $K+1$ of cost $B$ is added  to the bandit instance, with a positive
probability of  pulling arm $K+1$ representing  the decision not to  spend the
entire  budget on  pulling  arms  $1,\ldots,K$.  Though  both  the KL-UCB  and
Thompson  Sampling   algorithms  were   presented  without  this   extra  arm,
we already noted that % one readily sees that,
for each  $t$, $q_{K+1}(t)= 1-\frac{1}{B}\sum_{a=1}^K \cc_a  q_a(t)$.  The UCB
index $U_{K+1}(t)$ and posterior draw $\theta_{K+1}(t)$ for arm $K+1$ are both
equal to $B\rho$ for all $t$. For the sake of condensing notation in our study
of  (expected)  regret, it  will  be  convenient  to consider  a  hypothetical
scenario in  which arm $K+1$ is  pulled with probability $q_{K+1}(t)$  at each
time  point, even  though the  outcome of  these pulls  has no  effect on  the
behavior of the algorithms.

\subsection{Suboptimal arms not pulled often} \label{sec:suboptrare}

In this section, we establish \eqref{eq:subopt} and \eqref{eq:suboptindifference} for KL-UCB and Thompson Sampling.

For a fixed arm $a$, the KL-UCB and Thompson sampling proofs will both rely on a quantity $\rho^\dagger\in(\rho_a,\mu_{+}/\cc_a)$, though we note that the value that we select for $\rho^\dagger$ will vary between the proofs.

\subsubsection*{KL-UCB}

\paragraph{\em Preliminary: a general analysis.} 
We start by giving a general analysis of KL-UCB in our setting, and then use it to prove Theorems~\ref{thm:expfam}~and~\ref{thm:finsup}. 
Fix  $a\in \cN\backslash\{K+1\}$. 
The arguments in this section generalize those given in \cite{Cappeetal2013,Cappeetal2013b} for the case where one arm is drawn at each time point and there is no budget constraint. 
Let $\mu^\dagger\in(\mu_a,\mu_{+})$ be some real number. If $a\in\underline{\cN}$, then we will choose $\mu^\dagger=\cc_a \rho^\star$. If, on the other hand, $a\in\overline{\cN}$, then we will choose $\mu^\dagger$ to be less than $\mu_{+}$. Let $\rho^\dagger$ be a constant that is either equal to or slightly less than $\mu^\dagger/\cc_a$. Below we take minimums over $a^\star\in\cS\equiv (\cL\cup \cM)\backslash\{K+1\}$: if $\cS=\emptyset$, then we take these minimums to be equal to negative infinity. When we later take sums over $a^\star\in\cS$, we let empty sums equal zero. 

We now establish that, for all $t\ge K$,
\begin{align}
\left\{a\in\A(t+1)\right\}&\subseteq \left[\cup_{a^\star\in\cS}\left\{\cc_{a^\star} \rho^\dagger\ge U_{a^\star}(t)\right\}\right]\cup \left\{a\in\A(t+1), \cc_a \rho^\dagger< U_a(t)\right\}.\label{eq:Atp1a}
\end{align}
We separately handle the cases that $\rho^\star>\rho$ and $\rho^\star=\rho$. If $\rho^\star > \rho$, playing all of the arms in $\cS$ would spend at least the allotted budget $B$. Hence, on the event 
$\left\{\forall a^\star\in \cS,  U_{a^\star}(t)/c_{a^\star} > \rho^\dagger\right\}$, it holds that $\hat{\rho}^\star(t) > \rho^\dagger$. If moreover $a\in\A(t+1)$, one has $U_a(t) \geq c_a\hat{\rho}^\star(t) > c_a\rho^\dagger$. 
If $\rho = \rho^\star$, it holds that $\{a \in \hat{\cA}(t+1) \} \subseteq \{a \in \hat{\cA}(t+1) , c_a \rho^\dagger < U_a(t)\}$. Indeed, if $\hat{\rho}^\star(t) > \rho$ the algorithm only pulls arms $a$ if $U_a(t)  \geq \hat{\rho}^\star(t) c_a > \rho c_a$ and if $\hat{\rho}^\star(t) = \rho$, then the algorithm only pulls arm $a$ if $U_a(t) > c_a\rho$, see Footnote~\ref{foot:extraif}. As $\rho^\dagger$ is smaller or equal to $\rho^\star=\rho$, it follows that $U_a(t) > c_a\rho^\dagger$ in both cases.

For each $\zeta>0$ and $\tilde{\mu}< \mu_{+}$, we now introduce the set $\mathcal{C}_{\tilde{\mu},\zeta}$. In the setting of Theorem~\ref{thm:expfam},
\begin{align*}
\mathcal{C}_{\tilde{\mu},\zeta}&\equiv \left\{\nu' : \Supp[\nu']\subseteq\bar{I}\right\} \cap \left\{\nu' : \exists\,\mu\in(\tilde{\mu},\mu_{+}]\textnormal{ with }\KL(E(\nu'),\mu)\le \zeta\right\},
\end{align*}
where above $\KL(E(\nu'),\mu)$ is the KL-divergence in the canonical exponential family $\mathcal{E}$. In the setting of Theorem~\ref{thm:finsup},
\begin{align*}
\mathcal{C}_{\tilde{\mu},\zeta}&\equiv \left\{\nu' : \Supp[\nu']\subseteq[0,1]\right\} \cap \left\{\nu' : \exists\,\nu\in\mathcal{B}\textnormal{ with }\tilde{\mu}<E(\nu)\textnormal{ and }\KL(\Pi_{\cD}(\nu'),\nu)\le \zeta\right\}.
\end{align*}
In both settings, we will invoke this set at $\tilde{\mu}=\cc_a \rho^\dagger<\mu_{+}$. The set $\mathcal{C}_{\tilde{\mu},\zeta}$ is defined in both settings so that $\tilde{\mu}< U_a(t)$ if and only if $\hat{\nu}_a(t)\in \mathcal{C}_{\tilde{\mu},f(t)/N_a(t)}$. Recalling that $\E[N_a(T)]=\sum_{t=0}^{T-1} \Prob\{a\in\mathcal{A}(t+1)\}$, a union bound gives
\begin{align*}
\E[N_a(T)]\le&\, 1 + \sum_{a^\star\in\cS} \sum_{t=K}^{T-1} \Prob\left\{\cc_{a^\star} \rho^\dagger\ge U_{a^\star}(t)\right\} + \sum_{t=K}^{T-1} \Prob\left\{a\in\A(t+1), \hat{\nu}_{a,N_a(t)}\in \mathcal{C}_{\cc_a \rho^\dagger,f(t)/N_a(t)}\right\}.
\end{align*}
In analogue to Equation 8 in \cite{Cappeetal2013}, the above rightmost term satisfies
\begin{align}
\sum_{t=K}^{T-1} &\Prob\left\{a\in\A(t+1), \hat{\nu}_{a,N_a(t)}\in \mathcal{C}_{\cc_a \rho^\dagger,f(t)/N_a(t)}\right\} \nonumber \\
\le& \sum_{t=K}^{T-1}\Prob\left\{a\in\A(t+1), \hat{\nu}_{a,N_a(t)}\in \mathcal{C}_{\cc_a \rho^\dagger,f(T)/N_a(t)}\right\} \nonumber \\
=& \sum_{t=K}^{T-1} \sum_{n=2}^{T-K+1}\Prob\left\{\hat{\nu}_{a,n-1}\in \mathcal{C}_{\cc_a \rho^\dagger,f(T)/(n-1)},\tau_{a,n}=t+1\right\} \label{eq:taunKL} \\
\le& \sum_{n=1}^{T-K}\Prob\left\{\hat{\nu}_{a,n}\in \mathcal{C}_{\cc_a \rho^\dagger,f(T)/n}\right\}, \nonumber
\end{align}
where the final inequality holds because, for each $n$, $\tau_{a,n}=t+1$ for at most one $t$ in $\{K,\ldots,T-1\}$. We will upper bound the terms with $n=1,\ldots,b_a^{\star}(T)$ in the sum on the right by $1$, where
\begin{align*}
b_a^{\star}(T)&\equiv \left\lceil\frac{f(T)}{\Kinf(\nu_a,\mu^\dagger)}\right\rceil\le \frac{f(T)}{\Kinf(\nu_a,\mu^\dagger)}+1.
\end{align*}
This gives the bound
\begin{align*}
\sum_{n=1}^{T-K} \Prob\left\{\hat{\nu}_{a,n}\in\mathcal{C}_{\cc_a \rho^\dagger,f(T)/n}\right\}&\le \frac{f(T)}{\Kinf(\nu_a,\mu^\dagger)} + 1 + \sum_{n=b_a^{\star}(T) + 1}^{\infty} \Prob\left\{\hat{\nu}_{a,n}\in\mathcal{C}_{\cc_a \rho^\dagger,f(T)/n}\right\}.
\end{align*}
Hence,
\begin{align}
\E[N_a(T)]&\le \frac{f(T)}{\Kinf(\nu_a,\mu^\dagger)} + \underbrace{\sum_{n=b_a^{\star}(T) + 1}^{\infty} \Prob\left\{\hat{\nu}_{a,n}\in\mathcal{C}_{\cc_a \rho^\dagger,f(T)/n}\right\}}_{\textnormal{Term 1}} + \sum_{a^\star\in\cS} \underbrace{\sum_{t=K}^{T-1} \Prob\left\{\cc_{a^\star} \rho^\dagger\ge U_{a^\star}(t)\right\}}_{\textnormal{Term 2}a^\star} + 2. \label{eq:T1T2}
\end{align}
Up until this point we have not committed to any particular choice of $\mu^\dagger$, $\rho^\dagger$, or non-decreasing function $f : \mathbb{N}\rightarrow\mathbb{R}$. We now give proofs of \eqref{eq:subopt} and \eqref{eq:suboptindifference} in the settings of Theorems \ref{thm:expfam} and \ref{thm:finsup}. For each proof we use the choice of $f$ from the theorem statement and make particular choices of $\mu^\dagger$ and $\rho^\dagger$.

\begin{lemma}\label{lem:klucb1}
Eq.~\ref{eq:subopt} holds in the settings of Theorems \ref{thm:expfam} and \ref{thm:finsup}
\end{lemma}
\begin{proof}
Fix $a\in\cN\backslash\{K+1\}$. If $a\in\underline{\cN}$, then let $\mu^\dagger=\cc_a \rho^\star$ and, if $a\in\overline{\cN}$, then let $\mu^\dagger\in(\mu_a,\mu_{+})$. In the setting of Theorem~\ref{thm:expfam} let $\rho^\dagger=\mu^\dagger/\cc_a$ and in the setting of Theorem~\ref{thm:finsup} let $\rho^\dagger=\left[1-\log(T)^{-1/5}\right]\mu^\dagger/\cc_a$. Lemma~\ref{lem:term1} shows that Term 1 is $o(\log T)$ and includes references on where to find an explicit finite sample upper bound, where this upper bound will rely on the choice of $\mu^\dagger<\mu_{+}$ if $a\in\overline{\cN}$. Fix $a^\star\in\cS$. Noting that $\rho^\dagger\le \left[1-\log(T)^{-1/5}\right]\rho_{a^\star}$ (Theorem~\ref{thm:expfam}) and $\rho^\dagger\le \rho_{a^\star}$ (Theorem~\ref{thm:finsup}), Term 2$a^\star$ is $o(\log T)$ in both settings by Lemma~\ref{lem:term2astar}, with an exact finite sample upper bound given in the proof thereof. Thus, $\sum_{a^\star\in\cS}\textnormal{Term 2}a^\star = o(\log T)$. This completes the proof of \eqref{eq:subopt}. \qed
\end{proof}

\begin{lemma}\label{lem:klucb2}
Eq.~\ref{eq:suboptindifference} holds in the settings of Theorems \ref{thm:expfam} and \ref{thm:finsup}
\end{lemma}
\begin{proof}
For $a\in\overline{\cN}$, so far we have established that, for arbitrary $\mu^\dagger\in(\mu_a,\mu_{+})$,
\begin{align*}
\E[N_a(T)]&\le \frac{\log T}{\Kinf(\nu_a,\mu^\dagger)} + r(T,\mu^\dagger),
\end{align*}
where $r(T,\mu^\dagger)/\log T\rightarrow 0$ for fixed $\mu^\dagger$. As this holds for every $\mu^\dagger$, there exists a sequence $\mu^\dagger(T)\rightarrow\mu_{+}$ such that $r(T,\mu^\dagger(T))/\log T\rightarrow 0$. In both settings $\liminf_{\mu^\dagger\rightarrow\mu_{+}}\Kinf(\nu_a,\mu^\dagger)=+\infty$, and so using this $\mu^\dagger(T)$ sequence shows that $\E[N_a(T)]=o(\log T)$. \qed
\end{proof}

\subsubsection*{Thompson Sampling}

This proof is inspired by the analysis of Thompson sampling proposed by \cite{Agrawal&Goyal2012}. We work with a suboptimal arm $a\in\cN\backslash\{K+1\}$ in most of this section, though we state one of the results (Lemma~\ref{lem:suboptasopt}) for general arms $a\in\{1,\ldots,K+1\}$ since it will prove useful later. We will let $\rho^\dagger$ and $\rho^\ddagger$ be numbers (to be specified later) satisfying $\rho_a<\rho^\dagger<\rho^\ddagger<1/\cc_a$. Observe that $\left\{a\in\A(t+1)\right\}$ equals
\begin{align*}
&\left\{a\in\A(t+1),\theta_a(t)\le \cc_a\rho^\ddagger\right\}\cup \left\{a\in\A(t+1),\theta_a(t)>\cc_a\rho^\ddagger\right\} \\
&\subseteq \left[\cup_{a^\star\in\cL\cup\cM}\left\{a\in\A(t+1),\theta_a(t)\le\cc_a\rho^\ddagger,\theta_{a^\star}(t)\le \cc_{a^\star}\hat{\rho}^\star\right\}\right]\cup \left\{a\in\A(t+1),\theta_a(t)>\cc_a\rho^\ddagger\right\}.
\end{align*}
By the absolute continuity of the beta distribution, with probability one at most one $a'\in\{1,\ldots,K+1\}$ satisfies $\theta_{a'}(t)= \cc_{a'}\hat{\rho}^\star$, and hence, conditional on $\mathcal{F}(t)$, the leading event above is almost surely equivalent to the event
\begin{align*}
\cup_{a^\star\in\cL\cup\cM}&\left\{a\in\A(t+1),\theta_a(t)\le\cc_a\rho^\ddagger,\theta_{a^\star}(t)< \cc_{a^\star}\hat{\rho}^\star\right\}.
\end{align*}
If $K+1 \in\cM$, then the fact that $a\in\A(t+1)$ implies that $\theta_a(t)/\cc_a(t)\ge \hat{\rho}^\star(t)$ shows that the event in the union above at $a^\star=K+1$ never occurs, since on this event $\rho_{K+1}=\theta_{K+1}(t)/\cc_{K+1}<\rho^\ddagger$, which contradicts our choice that $\rho^\ddagger<\rho^\star=\rho_{K+1}$. Hence, the union above can be taken over $\cS$ regardless of whether or not $K+1\in\cM$. Furthermore,
\begin{align*}
&\left\{a\in\A(t+1),\theta_a(t)>\cc_a\rho^\ddagger\right\} \\
&\subseteq \left\{a\in\A(t+1),\theta_a(t)>\cc_a\rho^\ddagger,\hat{\mu}_a(t)\le\cc_a\rho^\dagger\right\}\cup \left\{a\in\A(t+1),\hat{\mu}_a(t)>\cc_a\rho^\dagger\right\}.
\end{align*}
Recalling that $\E[N_a(T)] = \sum_{t=0}^{T-1} \Prob\left\{a\in\A(t+1)\right\}$,
\begin{align}
\E[N_a(T)]\le&\, \sum_{a^\star\in\cS} \underbrace{\sum_{t=0}^{T-1}\Prob\left\{a\in\A(t+1),\theta_a(t)\le\cc_a\rho^\ddagger,\theta_{a^\star}(t)< \cc_{a^\star}\hat{\rho}^\star\right\}}_{\textnormal{Term I}a^\star} \nonumber \\
&+ \underbrace{\sum_{t=0}^{T-1} \Prob\left\{a\in\A(t+1),\theta_a(t)>\cc_a\rho^\ddagger,\hat{\mu}_a(t)\le\cc_a\rho^\dagger\right\}}_{\textnormal{Term II}} \nonumber \\
&+ \underbrace{\sum_{t=0}^{T-1} \Prob\left\{a\in\A(t+1),\hat{\mu}_a(t)>\cc_a\rho^\dagger\right\}}_{\textnormal{Term III}}. \label{eq:termsitoiii}
\end{align}
The above decomposition does not depend on the algorithm. Bounding Terms I$a^\star$, $a^\star\in\cS$, and Term II will rely on arguments that are specific to Thompson Sampling. Fix $a^\star\in \cS$ and let $p_{a^\star}^{\rho^\ddagger}(t)\equiv \Prob(\theta_{a^\star}(t)>\cc_{a^\star}\rho^\ddagger\,|\,\mathcal{F}(t))$.  Note that $p_{a^\star}^{\rho^\ddagger}(t)\not= p_{a^\star}^{\rho^\ddagger}(t+1)$ implies $a^\star\in\A(t+1)$. Thus $p_{a^\star}^{\rho^\ddagger}(t)$ is equal to $p_{a^\star,n}^{\rho^\ddagger}\equiv p_{a^\star}^{\rho^\ddagger}(\tau_{a^\star,n})$ for all $t$ such that $N_{a^\star}(t)=n$. We now state Lemma~\ref{lem:suboptasopt}, that  generalizes Lemma~1 in \cite{Agrawal&Goyal2012}.
\begin{lemma} \label{lem:suboptasopt}
If $a\in\{1,\ldots,K+1\}$, $a^\star\in\cS$, and $\rho^\ddagger$ satisfies $\cc_{a^\star}\rho^\ddagger<1$, then, for all $t\ge 0$,
\begin{align*}
\Prob\left(\left.a\in\A(t+1),\theta_a(t)\le\cc_a \rho^{\ddagger},\theta_{a^\star}(t)< \cc_{a^\star}\hat{\rho}^\star\right|\mathcal{F}(t)\right)&\le \frac{1-p_{a^\star}^{\rho^\ddagger}(t)}{p_{a^\star}^{\rho^\ddagger}(t)}\Prob\left(\left.\theta_{a^\star}(t)/\cc_{a^\star}\ge \hat{\rho}^{\star}(t)\right|\mathcal{F}(t)\right).
\end{align*}
\end{lemma}
The proof can be found in Appendix~\ref{app:thomproof}. 
Observe that the upper bound in the above lemma does not rely on $a$. We have another lemma, that relies on a lower bound on the probability $\mathring{q}_{a^\star}$, to be defined shortly, that is possible for $q_{a^\star}(t)$ given that $\theta_{a^\star}(t)/\cc_{a^\star}\ge \hat{\rho}^\star(t)$. By the absolute continuity of the beta distribution, we also have that
\begin{align*}
\Prob\left(a^\star\in\A(t+1)\middle|\mathcal{F}(t)\right)&= \Prob\left(a^\star\in\A(t+1),\frac{\theta_{a^\star}(t)}{\cc_{a^\star}}\ge \hat{\rho}^{\star}(t)\middle|\mathcal{F}(t)\right) \\
&= \Prob\left(a^\star\in\A(t+1)\middle|\frac{\theta_{a^\star}(t)}{\cc_{a^\star}}\ge \hat{\rho}^{\star}(t),\mathcal{F}(t)\right)\Prob\left(\frac{\theta_{a^\star}(t)}{\cc_{a^\star}}\ge \hat{\rho}^{\star}(t)\middle|\mathcal{F}(t)\right). 
\end{align*}
We lower bound the leading term in the product on the right by
\begin{align*}
\mathring{q}_{a^\star}&\equiv \min\left\{1,\min_{\mathcal{H}\subseteq \{1,\ldots,K\}\backslash\{a^\star\} : \sum_{\tilde{a}\in\mathcal{H}}\cc_{\tilde{a}} < B} \frac{B-\sum_{\tilde{a}\in\mathcal{H}}\cc_{\tilde{a}}}{c_{a^\star}}\right\}.
\end{align*}
Because $\cc_{K+1}=B$, one could equivalently take the minimum over $\mathcal{H}\subseteq \{1,\ldots,K+1\}\backslash\{a^\star\}$. To see that this is a lower bound, consider two cases. If $\theta_{a^\star}(t)/\cc_{a^\star}> \hat{\rho}^{\star}(t)$, then $a\in\A(t+1)$ with probability one, and so the above is a lower bound. If $\theta_{a^\star}(t)/\cc_{a^\star}= \hat{\rho}^{\star}(t)$, then the numerator $B - \sum_{\tilde{a}\in\mathcal{H}}\cc_{\tilde{a}}$ of the inner minimum (over $\mathcal{H}$) above represents the minimum possible amount of remaining budget when arm $a^\star$ is the unique arm on the estimated margin. The estimated margin is almost surely (over the draws of $\theta(t)$) singleton. Clearly, $\mathring{q}_{a^\star}>0$. As a consequence,,
\begin{align}
\Prob\left(\frac{\theta_{a^\star}(t)}{\cc_{a^\star}}\ge \hat{\rho}^{\star}(t)\middle|\mathcal{F}(t)\right)&\le \mathring{q}_{a^\star}^{-1} \Prob\left(a^\star\in\A(t+1)\middle|\mathcal{F}(t)\right). \label{eq:calSvscalA}
\end{align}
We have the following lemma, whose proof can be found in Appendix~\ref{app:thomproof}.
\begin{lemma} \label{lem:thomtransferTtoN}
If $a^\star\in\cS$ and $\cc_{a^\star}\rho^\ddagger<1$, then, for all $t\ge 0$,
\begin{align*}
\E\left[\sum_{t=0}^{T-1} \frac{1-p_{a^\star}^{\rho^\ddagger}(t)}{p_{a^\star}^{\rho^\ddagger}(t)}\Prob\left(\left.\theta_{a^\star}(t)/\cc_{a^\star}\ge \hat{\rho}^{\star}(t)\right|\mathcal{F}(t)\right)\right]&\le \mathring{q}_{a^\star}^{-1}\E\left[\sum_{n=0}^{T-1} \frac{1-p_{a^\star,n}^{\rho^\ddagger}}{p_{a^\star,n}^{\rho^\ddagger}}\right].
\end{align*}
\end{lemma}
Combining the two preceding lemmas yield the inequality
\begin{align}
\textnormal{Term I}a^\star&\le \mathring{q}_{a^\star}^{-1} \E\left[\sum_{n=0}^{T-1} \frac{1-p_{a^\star,n}^{\rho^\ddagger}}{p_{a^\star,n}^{\rho^\ddagger}}\right]. \label{eq:suboptasoptsum}
\end{align}
Note crucially that we have upper bounded the sum over time on the left-hand side by a sum over the number of pulls of arm $a^\star$ on the right-hand side. There appears to be a steep price to pay for this transfer from a sum over time to a sum over counts: the right-hand side inverse weights by a conditional probability, which may be small for certain realizations of the data. Lemma~2 in \cite{Agrawal&Goyal2012}, that we restate below using our modified notation, establishes that this inverse weighting does not cause a problem for Thompson sampling with Bernoulli rewards and independent beta priors. If $\rho^\ddagger<\rho^\star$, then the proceeding lemma implies that, for each $a^\star\in\cS$, Term I$a^\star$ is $O(1)$, i.e. is $o(\log T)$ with much to spare. Obviously, this implies that $\sum_{a^\star\in\cS}\textnormal{Term I}a^\star = o(\log T)$ as well.
\begin{lemma}[Lemma~2 from \citealp{Agrawal&Goyal2012}] \label{lem:carefulbinomialbound}
If $a^\star\in\cS$ and $\rho^\ddagger<\rho_{a^\star}$, then, with $\Delta\equiv \mu_{a^\star}-\cc_{a^\star} \rho^\ddagger$,
\begin{align*}
\E\left[\frac{1-p_{a^\star,n}^{\rho^\ddagger}}{p_{a^\star,n}^{\rho^\ddagger}}\right]&=\begin{cases}
\frac{3}{\Delta},&\mbox{ for }n<\frac{8}{\Delta} \\
\Theta\left(e^{-\Delta^2 n/2} + \frac{1}{(n+1)\Delta^2}e^{-\KL(\cc_{a^\star}\rho^\ddagger,\mu_{a^\star})n} + \frac{1}{\exp(\Delta^2 n/4)-1}\right),&\mbox{ for }n\ge \frac{8}{\Delta}.
\end{cases}
\end{align*}
Above $\Theta(\cdot)$ is used to represent big-Theta notation.
\end{lemma}

We now turn to Term II. The following result mimics Lemma~4 in \cite{Agrawal&Goyal2012}, and is a consequence of the close link between beta and binomial distributions and the Chernoff-Hoeffding bound. We provide a proof of this result in Appendix~\ref{app:thomproof}.
\begin{lemma} \label{lem:termiii}
If $a\in(\cM\cup\cN)\backslash\{K+1\}$ and $\rho_a<\rho^\dagger<\rho^\ddagger$, where $\cc_a\rho^\ddagger< 1$, then
\begin{align*}
\textnormal{Term II}\equiv\sum_{t=0}^{T-1}\Prob\left\{a\in\A(t+1),\theta_a(t)>\cc_a\rho^\ddagger,\hat{\mu}_a(t)\le\cc_a\rho^\dagger\right\}&\le \frac{\log T}{\KL(\cc_a\rho^\dagger,\cc_a\rho^\ddagger)}.
\end{align*}
\end{lemma}

We now turn to Term III. Note that
\begin{align}
\textnormal{Term III}&= \E\left[\sum_{t=0}^{T-1} \Ind\left\{a\in\A(t+1),\hat{\mu}_{a,N_{a}(t)}>\cc_a \rho^\dagger\right\}\right] \nonumber \\
&= \E\left[\sum_{t=0}^{T-1} \sum_{n=0}^{T-1} \Ind\left\{\tau_{a,n+1}=t+1,\hat{\mu}_{a,n}>\cc_a \rho^\dagger\right\}\right] \nonumber \\
&\le \sum_{n=0}^{T-1} \Prob\left\{\hat{\mu}_{a,n}>\cc_a \rho^\dagger\right\}, \label{eq:termii}
\end{align}
where the latter inequality holds because $\tau_{a,n+1}=t+1$ for at most one $t$ in $\{0,\ldots,T-1\}$. The following lemma controls the right-hand side of the above.
\begin{lemma} \label{lem:termii}
Fix an arm $a\in\{1,\ldots,K\}$. If $\rho^\dagger>\rho_a$ and $\cc_a \rho^\dagger<1$, then
\begin{align*}
\sum_{n=0}^{T-1} \Prob\left\{\hat{\mu}_{a,n}>\cc_a \rho^\dagger\right\}&\le 1 + \frac{1}{\KL(\cc_a \rho^\dagger,\mu_a)}.
\end{align*}
\end{lemma}
The proof is omitted, but is an immediate consequence of the Chernoff-Hoeffding bound and the additional bounding from the proof of Lemma~3 in \cite{Agrawal&Goyal2012}. Thus we have shown that Term III is $o(\log T)$, with much to spare as well.

The proof of \eqref{eq:subopt} and \eqref{eq:suboptindifference} in the setting of Theorem~\ref{thm:thom} is now straightforward.

\begin{lemma}\label{lemma:thom1}
Eq.~\ref{eq:subopt} holds in the setting of Theorem \ref{thm:thom}.
\end{lemma}
\begin{proof}
Fix $a\in\cN\backslash\{K+1\}$. Let $\mu^\dagger= \cc_a \rho^\star$ if $a\in\underline{\cN}$, and let $\mu^\dagger$ be slightly less than $\mu_{+}$ if $a\in\overline{\cN}$. Fix $\rho^\dagger<\rho^\ddagger$ and $\rho^\ddagger$ (to be specified shortly) so that $\rho_a<\rho^\dagger<\rho^\ddagger<\mu^\dagger/\cc_a$ and $\epsilon\in(0,1]$ a constant. Plugging our results on each Term I$a^\star$ and on Terms II and III into \eqref{eq:termsitoiii} then yields that
\begin{align*}
\E\left[N_a(T)\right]&\le \frac{\log T}{\KL(\cc_a \rho^\dagger,\cc_a \rho^\ddagger)} + 1 + \frac{1}{\KL(\cc_a \rho^\dagger,\mu_a)} + O(1).
\end{align*}
Select $\rho^\dagger$ so that $\KL(\cc_a \rho^\dagger,\mu^\dagger)=\frac{\KL(\mu_a,\mu^\dagger)}{1+\epsilon}$ and $\rho^\ddagger$ so that $\KL(\cc_a \rho^\dagger,\cc_a \rho^\ddagger)=\frac{\KL(\cc_a \rho^\dagger,\mu^\dagger)}{1+\epsilon}$, since this gives $\KL(\cc_a \rho^\dagger,\cc_a \rho^\ddagger) = \frac{\KL(\mu_a,\mu^\dagger)}{(1+\epsilon)^2}$. Hence,
\begin{align*}
\E[N_a(T)]\le (1+\epsilon)^2\frac{f(T)}{\KL(\mu_a,\mu^\dagger)} + r(T,\mu^\dagger),
\end{align*}
where $r(T,\mu^\dagger)/\log T\rightarrow 0$ for fixed $\mu^\dagger$. \qed
\end{proof}

\begin{lemma}\label{lemma:thom2}
Eq.~\ref{eq:suboptindifference} holds in the setting of Theorem \ref{thm:thom}.
\end{lemma}
\begin{proof}
If $a\in\underline{\cN}$, then dividing both sides by $\log T$, and then taking $T\rightarrow\infty$ followed by $\epsilon\rightarrow 0$ gives \eqref{eq:subopt}. If, on the other hand, $a\in\overline{\cN}$, then we use that there exists a sequence $\mu^\dagger(T)$ such that $r(T,\mu^\dagger(T))/\log T\rightarrow 0$. Because $\liminf_{\mu^\dagger\rightarrow\mu_{+}}=+\infty$, then dividing both sides by $\log T$, taking the limit as $T\rightarrow\infty$, followed by $\epsilon\rightarrow 0$, gives \eqref{eq:suboptindifference} in the case where $a\in\overline{\cN}$. \qed
\end{proof}

\subsection{Budget saturation when $\rho^\star > \rho$}\label{sec:budgetsat}

Assuming $\rho^\star > \rho$, we prove \eqref{eq:exhaustbudget} for KL-UCB and Thompson Sampling in the setting of Theorems \ref{thm:expfam} and \ref{thm:finsup} and Theorem~\ref{thm:thom} respectively. Recall that the third term in the regret decomposition \eqref{eq:regretDecomp2} can be expressed in terms of the number of draws of the supplementary arm $K+1$ in the extended bandit model: 
\[BT - \sum_{a=1}^Kc_a \bE_{\nu}[N_a(T)] = B \bE[N_{K+1}(T)].\]
We prove below for each algorithm that $\bE[N_{K+1}(T)] = o(\log(T))$, as a by product from specific elements already established when controlling the number of suboptimal draws. 

\subsubsection*{KL-UCB}

For any $\rho^\dagger\in(\rho,\rho^\star]$ and any $t\ge K$, it holds that, for $T$ large enough,  
\[\left\{K+1\in\A(t+1)\right\}\subseteq \bigcup_{a^\star\in\cS}\left\{\cc_{a^\star}\rho\ge U_{a^\star}(t)\right\}\subseteq \bigcup_{a^\star\in\cS}\left\{\cc_{a^\star}\rho^\dagger\ge U_{a^\star}(t)\right\}.\] 
The first inclusion must hold because if all the arms in $\cS$ had satisfied $U_{a^\star}/c_{a^\star} \geq \rho$, then including all of those arms in $\A(t+1)$ would have been enough to saturate the budget and $K+1$ would not have been selected. The second inclusion holds because $\rho^\dagger>\rho$. Hence, $\E[N_{K+1}(T)]\le \sum_{a^\star\in\cS} \textnormal{Term 2}a^\star$ (see Equation~\ref{eq:T1T2} for its definition). This condition is always satisfied by the choice $\rho^\dagger=\rho^\star$ that we have used in the setting of Theorem~\ref{thm:expfam}, and it holds for all $T$ sufficiently large for the choice $\rho^\dagger=\left[1-\log(T)^{-1/5}\right]\rho^\star$ that we have used in the setting of Theorem~\ref{thm:finsup}. Lemma~\ref{lem:term2astar} shows that each Term 2$a^\star$ is again $o(\log T)$.

\subsubsection*{Thompson Sampling}

We have that \[\{K+1\in\A(t+1)\}\subseteq\bigcup_{a^\star\in\cS}\{K+1\in\A(t+1),\theta_{a^\star}(t)\le \cc_{a^\star} \hat{\rho}^\star\}.\] As $\rho < \rho^\star$ and $\theta_{K+1}(t)=\cc_{K+1}\rho$ with probability one, $\E[N_a(T)]\le \sum_{a^\star\in \cS} \textnormal{Term I}a^\star$ provided $\rho^\dagger\in (\rho,\rho^\star)$ (see Equation~\ref{eq:termsitoiii} for its definition). Thus, we can invoke Lemma~\ref{lem:suboptasopt} (that holds for $a=K+1$), followed by Lemmas~\ref{lem:thomtransferTtoN}~and~\ref{lem:carefulbinomialbound}, to show that $\E[N_{K+1}(T)]=O(1)$, and therefore is $o(\log T)$ with much to spare.

\subsection{Optimal arms away from margin pulled $T-o(\log T)$ times} \label{sec:optcommon}
We now show that the optimal arms away from the margin ($a^\star\in\cL$) are pulled often. We start by giving an analysis that applies to any algorithm that, to decide which arms to draw at time $t+1$, based on $\mathcal{F}(t)$ and possibly some external stochastic mechanism, defines indices $I_a(t)$, $a=1,\ldots,K+1$, and then defines the threshold $\hat{\rho}^\star(t)\equiv\rho^\star(c_aI_a(t) : a=1,\ldots,K+1)$, and, for all arms $a$ with $I_a(t)\not=\hat{\rho}^\star(t)$, assigns mass $q_a(t)=\Ind\{I_a(t) > \hat{\rho}^\star(t)\}$. The arms with $I_a(t)=\hat{\rho}^\star(t)$ are assumed to be drawn so that $\sum_{a=1}^{K+1} \cc_a q_a(t)=B$. We then specialize the discussion to KL-UCB and Thompson sampling, where $I_a(t)$ is respectively equal to $U_a(t)/\cc_a$ and $\theta_a(t)/\cc_a$. For the remainder of this section, we fix an optimal arm $a^\star\in\cL$. Observe that, for $t\ge K$ (KL-UCB) or $t\ge 0$ (Thompson sampling),
\begin{align*}
\left\{I_{a^\star}(t)< \hat{\rho}^\star(t)\right\}=\,& \cup_{a\in\cM\cup\cN}\left\{I_a(t)\ge\hat{\rho}^\star(t),I_{a^\star}(t)< \hat{\rho}^\star(t)\right\} \\
= &\left[\cup_{a\in(\cM\cup\cN)\backslash\{K+1\}}\left\{I_a(t)\ge\hat{\rho}^\star(t),I_{a^\star}(t)< \hat{\rho}^\star(t),I_{K+1}(t)<\hat{\rho}^\star(t)\right\}\right] \\
&\,\cup\left\{I_{K+1}(t)\ge\hat{\rho}^\star(t),I_{a^\star}(t)< \hat{\rho}^\star(t)\right\}.
\end{align*}
Recalling \eqref{eq:calSvscalA}, we see that, for Thompson sampling,
\begin{align}
T - \E[N_{a^\star}(T)]=\,& T - \E\left[\sum_{t=0}^{T-1} \Prob\left\{a^\star\in\A(t+1)\middle|\mathcal{F}(t)\right\}\right] \nonumber \\
=\,& \E\left[\sum_{t=0}^{T-1} \Prob\left\{a^\star\not\in\A(t+1)\middle|\mathcal{F}(t)\right\}\right] \nonumber \\
\le\,& \E\left[\sum_{t=0}^{T-1} \Prob\left\{I_{a^\star}(t)< \hat{\rho}^\star(t)\middle|\mathcal{F}(t)\right\}\right] \nonumber \\
\le\,& \sum_{a\in(\cM\cup\cN)\backslash\{K+1\}} \sum_{t=0}^{T-1} \Prob\left\{I_a(t)\ge\hat{\rho}^\star(t),I_{a^\star}(t)< \hat{\rho}^\star(t),I_{K+1}(t)<\hat{\rho}^\star(t)\right\} \nonumber \\
&+ \sum_{t=0}^{T-1} \Prob\left\{I_{K+1}(t)\ge\hat{\rho}^\star(t),I_{a^\star}(t)< \hat{\rho}^\star(t)\right\}, \label{eq:bdthom}
\end{align}
where the first inequality holds because $\{a\not\in\A(t+1)\}\subseteq\{I_{a^\star}(t)< \hat{\rho}^\star(t)\}$ and the second inequality holds by the preceding display. We have a similar identity for KL-UCB, though the identity is slightly different due to the initiation of each of the $K$ arms. Specifically,
\begin{align}
T - K + 1 - \E[N_{a^\star}(T)]\le\,& \sum_{a\in(\cM\cup\cN)\backslash\{K+1\}} \sum_{t=0}^{T-1} \Prob\left\{I_a(t)\ge\hat{\rho}^\star(t),I_{a^\star}(t)< \hat{\rho}^\star(t),I_{K+1}(t)<\hat{\rho}^\star(t)\right\} \nonumber \\
&+ \sum_{t=0}^{T-1} \Prob\left\{I_{K+1}(t)\ge\hat{\rho}^\star(t),I_{a^\star}(t)< \hat{\rho}^\star(t)\right\}. \label{eq:bdKL}
\end{align}
For $a\in\cM\cup\cN$, let $\mathscr{H}$ denote the collection of all subsets $\mathcal{H}$ of $\{1,\ldots,K\}\backslash\{a,a^\star\}$ for which $\sum_{\tilde{a}\in\mathcal{H}}\cc_{\tilde{a}} < B$. For $a\in\cM\cup\cN$, we then define
\begin{align*}
\check{q}_{a}^{a^\star}\equiv \begin{cases}
\min\left\{1,\min_{\mathcal{H}\in \mathscr{H}} \frac{B-\sum_{\tilde{a}\in\mathcal{H}}\cc_{\tilde{a}}}{\cc_a}\right\},&\mbox{ if }a=K+1\textnormal{ or Thompson Sampling,} \\
\min\left\{1,\min_{\mathcal{H}\in \mathscr{H}} \frac{B-\sum_{\tilde{a}\in\mathcal{H}}\cc_{\tilde{a}}}{\sum_{\tilde{a}\in\{1,\ldots,K\}\backslash[\mathcal{H}\cup\{a^\star\}]}\cc_{\tilde{a}}}\right\}&\mbox{ if }a\not=K+1\textnormal{ and KL-UCB.}
\end{cases}
\end{align*}
Above ``Thompson Sampling'' and ``KL-UCB'' in the conditioning statements refers to which of the two algorithms is under consideration. The latter condition represents the extreme scenario where the arms in $\tilde{a}\in\mathcal{H}$ have $I_{\tilde{a}}(t)>\hat{\rho}^\star(t)$, whereas the arms $\tilde{a}$ outside of $\mathcal{H}\cup\{a^\star,K+1\}$ have $I_{\tilde{a}}(t)=\hat{\rho}^\star(t)$. One can verify that $\check{q}_{a}^{a^\star}>0$. Similarly to \eqref{eq:calSvscalA}, for each $a\in(\cM\cup\cN)\backslash\{K+1\}$ and $t\ge K$ (KL-UCB) or $t\ge 0$ (Thompson sampling),
\begin{align*}
\Prob&\left\{I_a(t)\ge \hat{\rho}^\star(t),I_{a^\star}(t)< \hat{\rho}^\star(t),I_{K+1}(t)<\hat{\rho}^\star(t)\middle| \mathcal{F}(t)\right\} \le \frac{1}{\check{q}_{a}^{a^\star}} \Prob\left\{a\in\A(t+1),I_{a^\star}(t)< \hat{\rho}^\star(t)\middle| \mathcal{F}(t)\right\},
\end{align*}
and thus
\begin{align*}
\sum_{t=0}^{T-1} \Prob\left\{I_a(t)\ge\hat{\rho}^\star(t),I_{a^\star}(t)< \hat{\rho}^\star(t),I_{K+1}(t)<\hat{\rho}^\star(t)\right\}&\le \frac{1}{\check{q}_{a}^{a^\star}} \sum_{t=0}^{T-1} \Prob\left\{a\in\A(t+1),I_{a^\star}(t)< \hat{\rho}^\star(t)\right\}.
\end{align*}
For $a=K+1$, we similarly have
\begin{align*}
\sum_{t=0}^{T-1} \Prob\left\{I_{K+1}(t)\ge\hat{\rho}^\star(t),I_{a^\star}(t)< \hat{\rho}^\star(t)\right\}&\le \frac{1}{\check{q}_{a}^{a^\star}} \sum_{t=0}^{T-1} \Prob\left\{a\in\A(t+1),I_{a^\star}(t)< \hat{\rho}^\star(t)\right\}.
\end{align*}
For each $a\in\cM\cup\cN$, let
\begin{align*}
M_a^{a^\star}(T)\equiv\sum_{t=0}^{T-1} \Ind\{a\in\A(t+1),I_{a^\star}(t)< \hat{\rho}^\star(t)\}.
\end{align*}
The bounds \eqref{eq:bdKL} and \eqref{eq:bdthom} yield the key observation that we use in this section:
\begin{align}
\mbox{for KL-UCB:}&\ \ T - K + 1 - \E[N_{a^\star}(T)]\le \sum_{a\in\cM\cup\cN} \frac{1}{\check{q}_a^{a^\star}}\E[M_a^{a^\star}(T)]; \nonumber \\
\mbox{for Thompson sampling:}&\ \ T - \E[N_{a^\star}(T)]\le \sum_{a\in\cM\cup\cN} \frac{1}{\check{q}_a^{a^\star}}\E[M_a^{a^\star}(T)]. \label{eq:keyobs}
\end{align}
We note that, for most models $\cD_K$, there will generally not be a positive lower bound on $\check{q}_a^{a^\star}$ uniformly over distributions $\cV$ in $\cD_K$, where we note that the dependence of $\check{q}_a^{a^\star}$ on $\cV$ is suppressed in the notation. Therefore, on the one hand, if one were pursuing a worst-case analysis of the regret of our algorithms, where the maximal regret is studied over all $\cV\in\cD$, then it would typically not be possible to control the right-hand sides above. On the other hand, in our setting, in which we study the regret at a fixed $\cV$, it is true that $\check{q}_a^{a^\star}>0$, and so one can control the right-hand sides above provided they can control $\E[M_a^{a^\star}(T)]$ for arms $a\in\cM\cup\cN$. In what follows, we will show that we can indeed control $\E[M_a^{a^\star}(T)]$ for these arms.

Let $G$ be some integer in $[0,+\infty($ and $\delta\in(0,1)$ be a constant to be specified shortly. For convenience, we let $T^{(g)}\equiv\lfloor T^{(1-\delta)^g}\rfloor$ for $g\in\mathbb{N}$. We also define
\begin{align*}
&\overline{\cU}\equiv\left\{a\in (\cM\cup\cN)\backslash\{K+1\} : \cc_a\rho_{a^\star}\ge\mu_{+}\right\}, \\
&\underline{\cU}\equiv\left\{a\in (\cM\cup\cN)\backslash\{K+1\} : \cc_a\rho_{a^\star}<\mu_{+}\right\},
\end{align*}
where we note that $\overline{\cU}\cup\underline{\cU}=(\cM\cup\cN)\backslash\{K+1\}$. Our analysis relies on the following bound (for which we provide the arguments below):
\begin{align}
\sum_{a\in\cM\cup\cN} \E[M_a^{a^\star}(T)]\le\,& \E[N_{K+1}(T)] + \sum_{a\in\overline{\cU}} \E[M_a^{a^\star}(T)] + \sum_{a\in\underline{\cU}} \E[M_a^{a^\star}(T)] \nonumber \\
=\,&  o(\log T) + \underbrace{\sum_{a\in\overline{\cU}} \E[M_a^{a^\star}(T)] + \sum_{a\in\underline{\cU}} \E[M_a^{a^\star}(T^{(G)})]}_{\textnormal{Term A}} \nonumber \\
&+ \underbrace{\sum_{g=1}^G \sum_{a\in\underline{\cU}} \E[M_a^{a^\star}(T^{(g-1)})-M_a^{a^\star}(T^{(g)})]}_{\textnormal{Term B}}. \label{eq:margindecomp}
\end{align}
The inequality uses that $\E[M_{K+1}^{a^\star}(T)]\le \E[N_{K+1}(T)]$, and the equality holds using (i) a telescoping series and (ii) the fact that the algorithm achieves \eqref{eq:suboptindifference}: indeed, this was proven for both KL-UCB and Thompson sampling in Section~\ref{sec:budgetsat}.

We now present the key ingredients to bound Term A and B. Each lemma stated below holds for both KL-UCB in the settings of Theorems \ref{thm:expfam} and \ref{thm:finsup} and for Thompson sampling in the setting of Theorem~\ref{thm:thom}. Though these lemmas hold for both algorithms, the methods of proof for KL-UCB and for Thompson sampling are quite different. Thus we give the proofs of the lemmas in the settings of Theorems \ref{thm:expfam} and \ref{thm:finsup} in Appendix~\ref{app:klucbproof} and the proofs in the setting of Theorem~\ref{thm:thom} in Appendix~\ref{app:thomproof}.
\begin{lemma} \label{lem:TermA}
In the settings of Theorem~\ref{thm:expfam}, \ref{thm:finsup}, and \ref{thm:thom}, $\E[M_a^{a^\star}(T)]=o(\log T)$ for $a\in \overline{\cU}$ and, for fixed $G\ge 0$,
\begin{align*}
\E[M_a^{a^\star}(T^{(G)})]\le (1-\delta)^G\frac{\log T}{\Kinf(\nu_a,\cc_a \rho_{a^\star})}
\end{align*}
for $a\in\underline{\cU}$. As a consequence,
\begin{align*}
\textnormal{Term A}&\le (1-\delta)^G \sum_{a\in\underline{\cU}}\frac{\log T}{\Kinf(\nu_a,\cc_a \rho_{a^\star})} + o(\log T).
\end{align*}
\end{lemma}
The proof of Lemma~\ref{lem:TermA} borrows a lot from the proofs of \eqref{eq:subopt} and \eqref{eq:suboptindifference} for each algorithm. %Note that an exact finite sample upper bound on the $o(\log T)$ term in the settings of Theorems \ref{thm:expfam} and \ref{thm:finsup} can be found via the use of Lemmas \ref{lem:term1} and \ref{lem:term2astar} in the appendix.

Controlling Term B relies on a careful choice of $\delta>0$, which is specified in Lemma~\ref{lem:TermB} below. The proof of this lemma is highly original: indeed we first prove that the considered algorithm is uniformly efficient, which allows to exploit the lower bound \eqref{eq:problb} given in Theorem~\ref{thm:reglb}. Its proof is provided in the appendix for both KL-UCB and Thompson Sampling, and we sketch it below.

\begin{lemma} \label{lem:TermB} Let $d\in (0,1)$ and $\delta$ chosen such 
\begin{align}
\delta = d\left[1-\left(\max_{a\in\cN\cap\underline{\cU}}\frac{\Kinf(\nu_a,\cc_a\rho^\star)}{\Kinf(\nu_a,\cc_a\rho_{a^\star})}\right)^{1/2}\right], \label{def:ChoiceDelta}
\end{align}
and $\delta=d$ if $\cN\cap\underline{\cU} = \emptyset$. Then in the setting of Theorems \ref{thm:expfam}, \ref{thm:finsup}, and \ref{thm:thom}, Term B is $o(\log T)$. \end{lemma}

\begin{proof}[Sketch of proof of Lemma~\ref{lem:TermB}]
We first show that the algorithms are uniformly efficient in the sense defined in Section~\ref{sec:lb}. This result is an immediate consequence of the results in Section~\ref{sec:suboptrare}, which show that the arms in $\cN\backslash\{K+1\}$ are not pulled too often, plus the preliminary results in this section, which show that arms in $\cL$ are pulled often.
\begin{lemma} \label{lem:cons}
KL-UCB is uniformly efficient in the settings of Theorems~\ref{thm:expfam} and \ref{thm:finsup} and Thompson sampling is uniformly efficient in the setting of Theorem~\ref{thm:thom}.
\end{lemma}
\begin{proof}
Fix an arbitrary reward distribution $\mathcal{V}$. By by Lemma~\ref{lem:TermA} and the already proven \eqref{eq:subopt} and \eqref{eq:suboptindifference} in the settings of Theorems \ref{thm:expfam}, \ref{thm:finsup}, and \ref{thm:thom} (see Lemmas~\ref{lem:klucb1}, \ref{lem:klucb2}, \ref{lemma:thom1}, and \ref{lemma:thom2}), both of which hold for $\mathcal{V}$,
\begin{align*}
T-\E_{\mathcal{V}}[N_{a^\star}(T)]&\le \sum_{a\in\cM\cup\cN}\frac{1}{\check{q}_a^{a^\star}}\E_{\mathcal{V}}[M_a^{a^\star}(T)] + O(1)\\
&\le o(\log T) + \sum_{a\in\overline{\cU}}\frac{1}{\check{q}_a^{a^\star}}\E_{\mathcal{V}}[M_a^{a^\star}(T)] + \sum_{a\in\underline{\cU}}\frac{1}{\check{q}_a^{a^\star}}\E_{\mathcal{V}}[M_a^{a^\star}(T)] + O(1)
\end{align*}
for any $a^\star\in\cL$, where the $O(1)$ term is equal to zero for Thompson sampling and, by \eqref{eq:keyobs}, is $K-1$ for KL-UCB. The right-hand side is $O(\log T)$ by applying the results of Lemma~\ref{lem:TermA} to control the sums over $\overline{\cU}$ and $\underline{\cU}$. Section~\ref{sec:suboptrare} showed that arms in $\cN$ are not pulled often (at most $O(\log T)$ times). By \eqref{eq:regretDecomp2}, it follows that $R(T)=O(\log T)$, which is $o(T^{\alpha})$ for any $\alpha>0$.
\qed\end{proof}

Fix $g\in\mathbb{N}$ and an arm $a\in\cN\cap\underline{\cU}$. By the uniform efficiency of the algorithm established in Lemma~\ref{lem:cons}, we will be able to apply \eqref{eq:problb} from Lemma~\ref{lem:lb} to show that $N_a(T^{(g)})\ge (1-\delta)\frac{\log T^{(g)}}{\Kinf(\nu_{a},\cc_a \rho^\star)}$ with probability approaching 1. For now suppose this holds almost surely (in the proofs we deal with the fact that this happens with probability approaching rather than exactly 1). Our objective will be to show that this lower bound on $N_a(T^{(g)})$ suffices to ensure that $M_a^{a^\star}(T^{(g-1)})-M_a^{a^\star}(T^{(g)})$ is $o(\log T)$, in words that arm $a$ is pulled while arm $a^\star$ is pulled with probability zero ($I_{a^\star}(t)<\hat{\rho}^\star(t)$) at most $o(\log T)$ times from time $t=T^{(g)},\ldots,T^{(g-1)}$.  
 
 We will see that $\frac{\log T^{(g-1)}}{\Kinf(\nu_a,\cc_a\rho_{a^\star})}$ pulls of arm $a$ by time $T^{(g)}$ suffices to ensure this in both settings. Using that $(1-\delta)\log T^{(g)}\approx (1-\delta)^2\log T^{(g-1)}$, it will follow that we can control the sum in Term B for each $a\in\cN$ provided we choose $\delta\in(0,1)$ so that
\begin{equation}
(1-\delta)^2\frac{1}{\Kinf(\nu_a,\cc_a\rho^\star)}> \frac{1}{\Kinf(\nu_a,\cc_a\rho_{a^\star})}\textnormal{ for all $a\in\cN$.}\label{eq:deltalogic}
\end{equation}
It is easy to check to for any $d\in (0,1)$, $\delta$ as defined in Lemma~\ref{lem:TermB} satisfies this inequality. Note that $\Kinf(\nu_a,\cc_a\rho_{a^\star})\ge\Kinf(\nu_a,\cc_a\rho^\star)$, and thus $\delta\in(0,1)$. 
So far we have only considered suboptimal arms $a\in\cN\cap\underline{\cU}$, but the fact that, for any $a\in\cM\cap\underline{\cU}$, Lemma~\ref{lem:lb} ensures that $N_a(T^{(g)})> \log T^{(g)}/\epsilon$ with probability approaching 1 for \textit{any} $\epsilon>0$ shows that $a\in\cN\cap\underline{\cU}$ is indeed the harder case. Indeed, this is what we see in our proofs controlling Term B for the two algorithms.
\qed\end{proof}

We now conclude the analysis. Combining Equations \eqref{eq:keyobs} and \eqref{eq:margindecomp} with the bounds on Term A and B obtained in Lemma~\ref{lem:TermA} and Lemma~\ref{lem:TermB} yield, for any finite $G$ and for the particular choice of $\delta \in (0,1)$ given in \eqref{def:ChoiceDelta}
\begin{align}
\limsup_T \frac{T-\E[N_{a^\star}(T)]}{\log T}&\le (1-\delta)^G \sum_{a\in \underline{\cU}} \frac{1}{\check{q}_a^{a^\star} \Kinf(\nu_a,\cc_a\rho_{a^\star})}. \label{eq:fixedG}
\end{align} 
Taking $G$ to infinity yields the result. 
\section{Conclusion} \label{sec:conc}
We have established the asymptotic efficiency of KL-UCB and Thompson sampling for budgeted multiple-play bandit problem in which the cost of pulling each arm is known and, in each round, the agent may use any strategy for which the expected cost is no more than their budget. We have also introduced a pseudo-arm so that the agent has the option of reserving the remainder of their budget if the remaining arms have reward-to-cost ratios that fall below a prespecified indifference point. Thompson sampling outperforms KL-UCB in three of our four simulations scenarios. Despite the strong performance of Thompson sampling for Bernoulli rewards, we have been able to prove stronger results about KL-UCB in this work, dealing with more general distributions. Understanding for which distributions one of these algorithms is preferable to the other is an interesting area for future work.

All of the proofs in this work can handle the case that the set of optimal arms is not unique. In an earlier work, \cite{Komiyamaetal2015} established the optimality of Thompson sampling under a multiple play bandit model in which the set of optimal arms is unique. A potential area for future work would be to extend their arguments to the special case of our budgeted bandit setting in which the set of optimal arms is unique -- it would be interesting to see if their technique yields a shorter proof in this special case.

In future work, it would be interesting to consider an extension of our setting where the budget ($B_t$), indifference points ($\rho_t$), and costs ($\cc_t$) are random over time according to some exogeneous source of randomness. If only the budget is random over time, then, under some regularity conditions, the regret lower bound and regret of our algorithms would seem to be driven by the behavior of our algorithm for the fixed budget representing the upper edge of the support for the random budget, since this is the setting in which the most information is learned about the arm distributions (arms that are otherwise suboptimal can be optimal in this setting). If only the indifference point is variable over time, then the behavior of our algorithm will similarly be driven by the lowest indifference point, since the most information is available in this case. Combinations of variable budgets and indifference points will result in a similar analysis. Variable but known costs are more complex, because they have the potential to change the order and indices of the optimal arms. For sufficiently variable costs, we in fact expect that all arms will be pulled more than order $\log T$ times, since all arms will be optimal for certain cost realizations. Therefore, a careful study of a variable cost budgeted bandit problem may require very different techniques than those used in this work.

 \paragraph{Structure of the Supplementary Material}
 {\small Appendix~\ref{proofs:Oracle} focuses on  oracle strategy  and regret  decomposition.  Appendix~\ref{app:lbproof} contains  proofs establishing  the asymptotic  lower  bound on  the number  of
 suboptimal arm draws.  Appendices~\ref{app:klucbproof} and \ref{app:thomproof}
 contain technical proofs for KL-UCB and the Thompson sampling, respectively. %The
 %\texttt{R} \citep{R2014} code for running  one repetition of our simulation is also 
 %available in the Supplementary Materials.
 }

\paragraph{Acknowledgements}
{\small The authors acknowledge the support of the French Agence Nationale de la Recherche (ANR), under grant ANR-13-BS01-0005 (project SPADRO) and ANR-16-CE40-0002 (project BADASS). Alex Luedtke gratefully acknowledges the support of a Berkeley Fellowship.}

\newpage

\vskip 0.2in
\def\bibfont{\footnotesize}
\bibliography{persrule}

\newpage
\appendix
% The next two lines will reset the equation number to zero and then use A.1, A.2, etc. in Appendix to number equations
\setcounter{equation}{0}
\renewcommand{\theequation}{A.\arabic{equation}}
% Ditto for theorems
\setcounter{theorem}{0}
\renewcommand{\thetheorem}{A.\arabic{theorem}}
\renewcommand{\thecorollary}{A.\arabic{theorem}}
\renewcommand{\thelemma}{A.\arabic{theorem}}
\renewcommand{\theproposition}{A.\arabic{theorem}}
\renewcommand{\theconjecture}{A.\arabic{theorem}}

\section*{Appendix}
We begin with an outline of the results proven in this appendix and how they are related to one another. Lemma~\ref{lem:lb} gives a lower bound on the number of draws of each suboptimal arm  for a uniformly efficient algorithm. Deduced from Lemma~\ref{lem:lb}, Theorem~\ref{thm:reglb} gives an asymptotic regret lower bound \eqref{eq:reglb} for a uniformly efficient algorithm. The asymptotic lower bound is achieved whenever the expected number of draws of each suboptimal arm satisfies the appropriate asymptotic condition, either \eqref{eq:subopt} or \eqref{eq:suboptindifference} depending on the arm, and the expected number of draws of each optimal arm away from the margin satisfies the asymptotic condition \eqref{eq:nonmargin}. Theorems~\ref{thm:expfam} and \ref{thm:finsup} state that the variants of KL-UCB are uniformly efficient and achieve \eqref{eq:reglb} for rewards sampled either from a single parameter exponential family or from bounded and finitely supported distributions. Theorem~\ref{thm:thom} states that Thompson sampling is uniformly efficient and achieves \eqref{eq:reglb} for Bernoulli distributed rewards.

The first step of the proof of Theorems~\ref{thm:expfam}, \ref{thm:finsup}, and \ref{thm:thom} consists in showing that KL-UCB and Thompson sampling achieve the asymptotically optimal expected number of suboptimal arm draws, i.e. that \eqref{eq:subopt} and \eqref{eq:suboptindifference} hold in their contexts. For KL-UCB, this is a consequence of a preliminary analysis given in Lemmas~\ref{lem:term1} and \ref{lem:term2astar}. For Thompson sampling, this is a consequence of another preliminary analysis given in Lemmas~\ref{lem:suboptasopt} through \ref{lem:termiii}. The proof of Lemma~\ref{lem:termiii} relies on a link between the beta and binomial distributions given in Lemma~\ref{lem:betabinomial}.

The second step of the proof of Theorems~\ref{thm:expfam}, \ref{thm:finsup}, and \ref{thm:thom} consists in showing that KL-UCB and Thompson sampling are uniformly efficient in their respective contexts. This is a consequence of yet another preliminary analysis, \eqref{eq:subopt}, \eqref{eq:suboptindifference}, and Lemma~\ref{lem:TermA}.

The third step of the proof of Theorems~\ref{thm:expfam}, \ref{thm:finsup}, and \ref{thm:thom} consists in showing that KL-UCB and Thompson sampling achieve the asymptotically optimal expected number of optimal draws away from the margin, i.e. that \eqref{eq:nonmargin} holds in their contexts. This is a consequence of the preliminary analysis undertaken in step two and of Lemmas~\ref{lem:TermA} and \ref{lem:TermB}. The proofs of Lemmas~\ref{lem:TermA} and \ref{lem:TermB} hinge on Lemmas~\ref{lem:suboptasopt} through \ref{lem:termiii}. The proof of Lemma~\ref{lem:TermB} also relies on Lemma~\ref{lem:betabinomial}.

The fourth and final step of the proof of Theorems~\ref{thm:expfam}, \ref{thm:finsup}, and \ref{thm:thom} boils down to applying Theorem~\ref{thm:reglb}.

\section{Oracle strategy and regret decomposition} \label{proofs:Oracle}

\subsection{Proof of Proposition~\ref{prop:Oracle}}
Recall that 
\[\bm q^\star \in \underset{{\bm q \in [0,1]^K}}{\argmax} \sum_{a=1}^K q_a (\mu_a - c_a \rho) \ \ \ \text{such that} \ \ \ \sum_{a=1}^K q_a c_a \leq B.\]
Introducing $c_{K+1} = B$ and $\mu_{K+1} = B\rho$, one can prove that $\bm q^\star$ coincides with the first $K$ components of $\bm q^\star_{K+1} \in [0,1]^{K+1}$, that is defined as the solution to 
\begin{equation}\bm q^\star_{K+1} \in \underset{\bm q \in [0,1]^{K+1}}{\argmax} \sum_{a=1}^{K+1} q_a (\mu_a - c_a \rho) \ \ \ \text{such that} \ \ \ \sum_{a=1}^{K+1} q_a c_a = B\label{opt:2}\end{equation}
and that the two optimization problems have the same value. This is because as $\mu_{K+1} - c_{K+1}\rho = 0$, the two objective functions coincide: \[f_K(\bm q) \equiv \sum_{a=1}^{K} q_a (\mu_a - c_a \rho) =\sum_{a=1}^{K+1} q_a (\mu_a - c_a \rho)  \equiv f_{K+1}(\bm q_{K+1})\]
and if $\bm q$ satisfies the first constraint, there exists $q_{K+1}$ such that $\bm q_{K+1} = (\bm q,q_{K+1})$ satisfies the second constraint: $\sum_{a=1}^{K+1} q_a c_a = B$ (as $c_{K+1}=B$). Conversely, if $\bm q_{K+1}$ satisfies the second constraint, its first $K$ marginals clearly satisfy the first constraint. 

The common value $M^\star$ of these two optimization problem, that is the maximal achievable reward, can be rearranged a bit, using that $\sum_{a=1}^{K+1} q_a c_a \rho = \rho B$:
 \[M^\star = \sum_{a=1}^{K+1} q_a^\star \mu_a - \rho B,\]
where $\bm q^\star_{K+1} \in [0,1]^{K+1}$ is the solution to
\begin{equation}\bm q^\star_{K+1} \in \underset{\bm q \in [0,1]^{K+1}}{\argmax} \sum_{a=1}^{K+1} q_a \mu_a \ \ \ \text{such that} \ \ \ \sum_{a=1}^{K+1} q_a c_a = B.\label{opt:3}\end{equation}

Now introduce
\[L^\star \equiv \sum_{a \in \cL} \mu_{a} + \rho^\star\left(B - \sum_{a \in \cL} c_{a}\right).\]
The optimal weights are also defined by 
\[\bm q^\star_{K+1} \in \argmin{\bm q \in [0,1]^{K+1}} \left[L^\star - \sum_{a=1}^{K+1} q_a \mu_a \right] \ \ \ \text{such that} \ \ \ \sum_{a=1}^{K+1} q_a c_a = B.\]
The new objective can be rewritten as follows, where the `virtual' arm $K+1$ that has characteristics $\mu_{K+1} = B\rho$ and $c_{K+1}=B$ is added to either the set $\cM$ (if $\rho_{K+1}=\rho=\rho^\star$) or $\cN$ (if $\rho_{K+1}=\rho < \rho^\star$). 
\begin{align*}
 &L^\star - \sum_{a=1}^{K+1} q_a \mu_a  =  \sum_{a \in \cL} \mu_{a} + \rho^\star\left(B - \sum_{a \in \cL} c_{a}\right) - \sum_{a \in \cL} c_a q_a \rho_a - \sum_{a \in \cM} c_a q_a \rho^\star  - \sum_{b \in \cN} c_b q_b \rho_b\\
 & =  \sum_{a \in \cL} c_a\rho_{a} + \rho^\star\left(B - \sum_{a \in \cL} c_{a}\right) - \sum_{a \in \cL} c_a q_a \rho_a - \rho^\star \left( B - \sum_{a \in \cL}c_aq_a - \sum_{b \in \cN}c_b q_b\right)  - \sum_{b \in \cN} c_b q_b \rho_b\\
 & =  \sum_{a \in \cL} c_a\underbrace{(\rho_a - \rho^\star)}_{> 0}(1-q_a) + \sum_{b \in \cN}c_b\underbrace{(\rho^\star - \rho_b)}_{> 0} q_b.
\end{align*}
This shows that the objective function is always non negative, and that it can actually be set to the zero by choosing weights that satisfy $q_a = 1$ for all $a\in \cL$ and $q_b = 0$ for all $b \in \cL$. 

It remains to justify that such a choice is indeed feasible for some choices of weights on the arms in the margin $\cM$. This margin is never empty, as in the case $\rho^\star = \rho$, it does contain the `pseudo-arm' mentioned above. By definition of the sets $\cL$ and $\cM$,
\[\sum_{a \in \cL} c_a < B \ \ \ \text{and} \ \ \ \sum_{a \in \cL \cup \cM} c_a \geq B\]
hence, the solution can be ``completed'' by putting weight on the margin such that $\sum_{a \in \cL}c_a + \sum_{a \in \cN}q_ac_a = B$.

If $\rho < \rho^\star$, then the arm $K+1$ belongs to $\cN$ and as such $q_{K+1} = 0$ and the first $K$ marginals indeed satisfy the statement of Proposition~\ref{prop:Oracle}, with a non-empty margin. If $\rho = \rho^\star$, our `extended' margin only contains arm $K+1$, while the original margin is empty. As such the only arms with non-zero weights among the first $K$ marginals are the arms in $\cL$, for which the weight is one.

\subsection{Proof of Proposition~\ref{prop:RegretDec}} 
\[\Reg(T,\cV) = \bE\left[\sum_{t=1}^T (G^\star  - G(t))\right] =  \bE\left[\sum_{t=1}^T (G^\star  - \sum_{a=1}^Kq_a(t)(\mu_a - c_a \rho))\right]\]
The proof follows from a rewriting of 
\begin{align*}
 G^\star  - \sum_{a=1}^Kq_a(t)(\mu_a - c_a \rho)  = &\sum_{a\in \cL}c_a \rho_a + \rho^\star\left(B - \sum_{a \in \cL} c_a\right) - B\rho - \sum_{a=1}^Kq_a(t)(\mu_a -c_a\rho) \\
  = & \sum_{a\in \cL}c_a \rho_a + \rho^\star\left(B - \sum_{a \in \cL} c_a\right) - B\rho - \sum_{a=1}^{K+1}q_a(t)(\mu_a -c_a\rho),
\end{align*}
where we define $\mu_{K+1}=\rho B$, $c_K = B$ and let $q_{K+1}(t)$ be such that $\sum_{a=1}^{K+1}q_a(t) c_a = B$. This is possible as $\sum_{a=1}^{K}q_a(t) c_a \leq B$ due to the soft budget constraints and $c_{K+1}=B$ and 
\[q_{K+1}(t) = \frac{B - \sum_{a=1}^K c_aq_a(t)}{B}.\]
Thus one can further write
\begin{align*}
 &G^\star  - \sum_{a=1}^Kq_a(t)(\mu_a - c_a \rho)   =  \sum_{a\in \cL}c_a \rho_a + \rho^\star\left(B - \sum_{a \in \cL} c_a\right)  - \sum_{a=1}^{K+1}q_a(t)\mu_a \\
 & \ \ \ = \sum_{a\in \cL}c_a \rho_a + \rho^\star\left(B - \sum_{a \in \cL} c_a\right)  - \sum_{a\in \cL}q_a(t)c_a\rho_a - \rho^\star\sum_{a\in \cM}q_a(t)c_a- \sum_{a\in \cN}q_a(t)c_a\rho_a - q_{K+1}(t)\rho B
\end{align*}
Using that 
\[\sum_{a\in \cM}q_a(t)c_a = B - \sum_{a\in \cL}q_a(t)c_a - \sum_{a\in \cM}q_a(t)c_a - q_{K+1}(t)B,\]
one obtains 
\begin{align*}
 & G^\star  - \sum_{a=1}^Kq_a(t)(\mu_a - c_a \rho)   =  \sum_{a\in \cL}c_a (\rho_a - \rho^\star)(1-q_a(t)) + \sum_{a \in \cN}c_a(\rho^\star - \rho_a) q_a(t) + B(\rho^\star - \rho) q_{K+1}(t)\\
 & \ \ \ = \sum_{a\in \cL}c_a (\rho_a - \rho^\star)(1-q_a(t)) + \sum_{a \in \cN}c_a(\rho^\star - \rho_a) q_a(t) + (\rho^\star - \rho)\left(B - \sum_{a=1}^K c_a q_a(t)\right)
\end{align*}
Summing over $t$, the regret can be decomposed as  
\[\sum_{a\in \cL}c_a (\rho_a - \rho^\star)\left(T- \bE\left[\sum_{t=1}^Tq_a(t)\right]\right) + \sum_{a \in \cN}c_a(\rho^\star - \rho_a) \bE\left[ \sum_{t=1}^T q_a(t)\right] + (\rho^\star - \rho)\left(B - \sum_{a=1}^K c_a \bE\left[\sum_{t=1}^T q_a(t)\right]\right)\]
and the conclusion follows by noting that $N_a(T) = \bE\left[\sum_{a=1}^T q_a(t)\right]$.

\section{Proof of Lower Bound on Suboptimal Arm Draws} \label{app:lbproof}
\begin{proof}[Proof of Lemma~\ref{lem:lb}]
Fix some arm $a\in(\cM\cup\underline{\cN})\backslash\{K+1\}$, natural number $T$, and $\delta\in(0,1)$. By definition, $\cc_a \rho^{\star}<\mu_{+}$ for all $a\in\underline{\cN}$, and, for $a\in\cM$ the same property holds by our assumption that $\cc_a\rho^\star=\mu_a<\mu_{+}$. Hence, the set $\{\tilde{\nu}_a\in\cD : E(\tilde{\nu}_a)>\cc_a \rho^{\star}\}$ is non-empty. If the intersection of this set with the set of distributions $\{\tilde{\nu}_a\in\cD : \nu_a \ll \tilde{\nu}_a\}$ is empty, then the bounds are trivial by our convention that $d/\infty=0$ for finite $d$. Otherwise, let $\mathcal{V}'$ be some distribution that is equal to $\mathcal{V}$ except in the $a^{\textnormal{th}}$ component, where its $a^{\textnormal{th}}$ component $\nu_a'\in\cD$ is such that $\mu_a'\equiv E(\nu_a')>\cc_a \rho^{\star}$ and $\nu_a \ll \nu_a'$. Furthermore, one can select $\mathcal{V}'$ to fall in the statistical model for the joint distribution of the arm-specific rewards by our variation-independence assumption. For each $b$, let $\rho_{b}'=\rho_{b}$, $b\not=a$, and let $\rho_a'=\mu_a'/\cc_a$. Observe that $\mu_a'>\cc_a\rho^{\star}\ge \mu_a$ implies that $\KL(\nu_a,\nu_a')>0$. Define the log-likelihood ratio random variable $L_a(T)\equiv L_{a,N_a(T)}\equiv \sum_{n=1}^{N_a(T)} \log\frac{d\nu_a}{d\nu_a'}(X_{a,n})$. Let $b_a(T)\equiv (1-\delta)\frac{\log T}{\KL(\nu_a,\nu_a')}$ and $d(T)\equiv (1-\delta/2)\log T$. We have that
\begin{align}
&\Prob_{\mathcal{V}}\left\{N_a(T)<b_a(T)\right\} \nonumber \\
&\le \Prob_{\mathcal{V}}\left\{N_a(T)<b_a(T),L_a(T)\le d(T)\right\} + \Prob_{\mathcal{V}}\left\{N_a(T)<b_a(T),L_a(T)> d(T)\right\} \nonumber \\
&\le e^{d(T)}\Prob_{\mathcal{V}'}\left\{N_a(T)<b_a(T)\right\} + \Prob_{\mathcal{V}}\left\{N_a(T)<b_a(T),L_a(T)> d(T)\right\}, \label{eq:lbdecomp}
\end{align}
where the final inequality holds because, for any event $D\subseteq\{N_a(T)=b,L_a(T)\le d(T)\}$, a change of measure shows that $\Prob_{\mathcal{V}}\{D\}=\bE_{\mathcal{V}'}\left[e^{L_{a,b}}\ind_{\{D\}}\right]\le e^{d(T)}\Prob_{\mathcal{V}'}\{D\}$ \citep[see Equation 2.6 in][]{Lai&Robbins1985}. Let $\tilde{\rho}^\star\equiv \rho^\star(\cc_a \rho_{a}' : a=1,\ldots,K+1)$. Observe that arm $a$ under the reward distribution involving $\nu_a'$ satisfies either (i) $\rho_{a'}'>\tilde{\rho}^\star$ or (ii) $\rho_a'=\tilde{\rho}^\star$ and $g_a\equiv B - \sum_{\tilde{a}\not=a : \rho_{\tilde{a}}\ge \tilde{\rho}^\star} \cc_{\tilde{a}}>0$, where the sum over the empty set is zero. 
% Using the uniform efficiency of the algorithm and the fact that arm $a$ under the reward distribution involving $\nu_a'$ is either (i) part of a unique set of optimal arms or (ii) has mean larger than the $m^{\textnormal{th}}$ largest arm mean,\footnote{For a general model $\cD$, (i) does not imply (ii) and (ii) does not imply (i). For example, if $\{E(\nu) : \nu\in \cD\}=\{0,1,2\}$, $K=3$, $m=2$, $a=3$, and $(\mu_1,\mu_2,\mu_3)=(2,1,0)$, then $\mathcal{V}'$ must have mean vector $(2,1,2)$ so that (i) holds and (ii) does not. In the same scenario but with $(\mu_1,\mu_2,\mu_3)=(1,1,0)$, then $\mathcal{V}'$ must have mean vector $(1,1,2)$ so that (ii) holds and (i) does not.}
Under (i), we note that the uniform efficiency of the algorithm and Markov's inequality yield that
\begin{align*}
\Prob_{\mathcal{V}'}\left\{N_a(T)<b_a(T)\right\}&= \Prob_{\mathcal{V}'}\left\{T-N_a(T)>T-b_a(T)\right\} = o\left(T^{\delta/2-1}\right).
\end{align*}
Thus, the first term in \eqref{eq:lbdecomp} converges to zero as $T\rightarrow\infty$ when (i) holds. We now show the same result when (ii) holds. We first note that
\begin{align*}
g_a T - \cc_a \E[N_a(T)]&\ge BT-\sum_{\tilde{a}\not=a : \rho_{\tilde{a}}\ge \tilde{\rho}^\star} \cc_{\tilde{a}}\E[N_{\tilde{a}}(T)] - \cc_a \E[N_a(T)]= BT-\sum_{\tilde{a} : \rho_{\tilde{a}}\ge \tilde{\rho}^\star} \cc_{\tilde{a}}\E[N_{\tilde{a}}(T)].
\end{align*}
The right-hand side is $o(T^{\delta/2})$ by the uniform efficiency of the algorithm. Hence, Markov's inequality yields that,
\begin{align*}
\Prob_{\mathcal{V}'}\left\{N_a(T)<b_a(T)\right\}&= \Prob_{\mathcal{V}'}\left\{g_aT-\cc_a N_a(T)>g_aT-\cc_a b_a(T)\right\} = o\left(T^{\delta/2-1}\right).
\end{align*}
Thus, the first term in \eqref{eq:lbdecomp} also converges to zero as $T\rightarrow\infty$ when (ii) holds. For the second term, observe that
\begin{align*}
\left\{N_a(T)<b_a(T),L_a(T)> d(T)\right\}&\subseteq \left\{\max_{n\le b_a(T)} \frac{L_{a,n}}{b_a(T)}> \frac{d(T)}{b_a(T)}\right\} \\
&= \left\{\max_{n\le b_a(T)} \frac{L_{a,n}}{b_a(T)}> \frac{1-\delta/2}{1-\delta}\KL(\nu_a,\nu_a')>\KL(\nu_a,\nu_a')\right\}.
\end{align*}
By the strong law of large numbers, $b_a(T)^{-1}L_{a,\lfloor b_a(T)\rfloor}\rightarrow \KL(\nu_a,\nu_a')$ almost surely under $\nu_a$. Further, $\max_{n\le b_a(T)} b_a(T)^{-1}L_{a,n}\rightarrow \KL(\nu_a,\nu_a')$ almost surely as $T\rightarrow\infty$. %(this fact is well known: see, e.g., Exercise 2.5.11 of \citet{Durrett10}).
It follows that the second term in \eqref{eq:lbdecomp} converges to zero as $T\rightarrow\infty$ so that
\begin{align}
\Prob_{\mathcal{V}}\left\{N_a(T)<(1-\delta) \frac{\log(T)}{\KL(\nu_a,\nu_a')}\right\}\rightarrow 0. \label{eq:problbnup}
\end{align}
For convenience, we let $\mathcal{K}\equiv \Kinf(\nu_a,\cc_a \rho^{\star})$ in what follows. By the definition of the infimum, for every $\epsilon>0$ there exists some $\nu_a'$ such that $\mathcal{K}+\epsilon>\KL(\nu_a,\nu_a')$. This proves \eqref{eq:problb}. If $a\in\underline{\cN}$ so that $\mathcal{K}>0$, then take $\epsilon = \left[(1-\delta)^{-1/2}-1\right]\mathcal{K}$ and write 
\begin{align*}
\Prob_{\mathcal{V}}\left\{N_a(T)<(1-\delta)^{3/2} \frac{\log(T)}{\Kinf(\nu_a,\cc_a \rho^\star)}\right\}\rightarrow 0. 
\end{align*}
Applying the above to $\delta' = 1 - (1-\delta)^{2/3}$ (such that $(1-\delta')^{3/2}=(1-\delta)$) yield the result for $a\in \underline{\cN}$. For $a\in \underline{\cN}$, it also follows that for all $\delta \in (0,1)$ one has  
\begin{align*}
\E[N_a(T)]\ge \frac{(1-\delta)\log T}{\Kinf(\nu_a,\cc_a \rho^\star)}\Prob_{\mathcal{V}}\left\{N_a(T)\ge (1-\delta) \frac{\log T }{\Kinf(\nu_a,\cc_a \rho^\star)}\right\} \underset{T \rightarrow \infty}{\sim}  \frac{(1-\delta)\log T}{\Kinf(\nu_a,\cc_a \rho^\star)},
\end{align*}
which yields \eqref{eq:explb}, letting $\delta$ go to zero. 
\qed\end{proof}

\section{Supplementary Proofs for KL-UCB} \label{app:klucbproof}
\begin{lemma} \label{lem:term1}
Fix an $a\in\{1,\ldots,K\}$ and a fixed $\mu^\dagger$ (not relying on $T$) with $\mu_a<\mu^\dagger$. In the setting of Theorem~\ref{thm:expfam} with $\rho^\dagger=\mu^\dagger/\cc_a$ or in the setting of Theorem~\ref{thm:finsup} with $\rho^\dagger=\left[1-\log(T)^{-1/5}\right]\mu^\dagger/\cc_a$, it holds that
\begin{align*}
\sum_{n=b(T) + 1}^{\infty} \Prob\left\{\hat{\nu}_{a,n}\in\mathcal{C}_{\cc_a \rho^\dagger,f(T)/n}\right\}&= o(\log T),
\end{align*}
where $b(T)$ is any number satisfying
\begin{align*}
b(T)&\ge \left\lceil\frac{f(T)}{\Kinf(\nu_a,\mu^\dagger)}\right\rceil.
\end{align*}
An explicit finite sample bound on the $o(\log T)$ term can be found in \cite{Cappeetal2013b}.
\end{lemma}
\begin{proof}
In the setting of Theorem~\ref{thm:expfam}, Equation 25 in \cite{Cappeetal2013b} gives the result for $\rho^\dagger=\mu^\dagger/\cc_a$. We refer the readers to that equation for the explicit finite sample bound that we are summarizing with little-oh notation.

In the setting of Theorem~\ref{thm:finsup}, Equation 33 combined with the unnumbered equation preceding Equation 36 in Section B.4 of \cite{Cappeetal2013b} gives the result for $\rho^\dagger=\left[1-\log(T)^{-1/5}\right]\mu^\dagger/\cc_a$. An explicit finite sample upper bound on this quantity can be found in Section B.4 of \cite{Cappeetal2013b}.
\qed\end{proof}

\begin{lemma} \label{lem:term2astar}
Fix an arm $a^\star\in\cS$. In the setting of Theorem~\ref{thm:expfam} with $\rho^\dagger\le\rho_{a^\star}$ or in the setting of Theorem~\ref{thm:finsup} with $\rho^\dagger\le\left[1-\log(T)^{-1/5}\right]\rho_{a^\star}$, it holds that
\begin{align*}
\sum_{t=K}^{T-1} \Prob\left\{\cc_{a^\star}\rho^\dagger\ge U_{a^\star}(t)\right\}&= o(\log T).
\end{align*}
Explicit finite sample constants can be found in the proof.
\end{lemma}
\begin{proof}
In the setting of Theorem~\ref{thm:expfam}, it holds that $\{\cc_{a^\star}\rho^\dagger\ge U_{a^\star}(t)\}\subseteq \{\mu_{a^\star}\ge U_{a^\star}(t)\}$. Hence,
\begin{align*}
\sum_{t=K}^{T-1} \Prob\left\{\cc_{a^\star}\rho^\dagger\ge U_{a^\star}(t)\right\}&\le \sum_{t=K}^{T-1} \Prob\{\mu_{a^\star}\ge U_{a^\star}(t)\}.
\end{align*}
Furthermore,
\begin{align*}
\{\mu_{a^\star}\ge U_{a^\star}(t)\}&\subseteq \bigcup_{n=1}^{t-K+1}\left\{\mu_{a^\star}\ge \hat{\mu}_{a^\star,n},\KL(\hat{\mu}_{a^\star,n},\mu_{a^\star})\ge\frac{f(t)}{n}\right\}.
\end{align*}
Using the above, Equations 17 and 18 in \cite{Cappeetal2013b} show that $\sum_{t=K}^{T-1} \Prob\{\mu_{a^\star}\ge U_{a^\star}(t)\}$ is upper bounded by $3 + 4e\log\log T = o(\log T)$ provided $T\ge 3$. %We note that the union above is over $n=\ell,\ldots,t-K+\ell$ rather than $n=1,\ldots,t-K+1$ as was used in Equation 17 of \cite{Cappeetal2013b}, but that the bound in their Equation 17 still holds in our setting.

In the setting of Theorem~\ref{thm:finsup}, it holds that $\{\cc_{a^\star}\rho^\dagger\ge U_{a^\star}(t)\}\subseteq \{\left[1-\log(T)^{-1/5}\right]\mu_{a^\star}\ge U_{a^\star}(t)\}$. Hence,
\begin{align}
\sum_{t=K}^{T-1} \Prob\left\{\cc_{a^\star}\rho^\dagger\ge U_{a^\star}(t)\right\}&\le \sum_{t=K}^{T-1} \Prob\left\{\left[1-\log(T)^{-1/5}\right]\mu_{a^\star}\ge U_{a^\star}(t)\right\}. \label{eq:PmudaggtUastar}
\end{align}
Let $\epsilon\equiv \log(T)^{-1/5} \mu_{a^\star}>0$. Arguments given in Section B.2 of \cite{Cappeetal2013b} show that
\begin{align*}
\left\{\mu_{a^\star}-\epsilon\ge U_{a^\star}(t)\right\}&\subseteq \left\{\Kinf(\hat{\nu}_{a^\star}(t),\mu_{a^\star}-\epsilon)\ge \frac{f(t)}{N_{a^\star}(t)}\right\} \\
&\subseteq \left\{\Kinf(\hat{\nu}_{a^\star}(t),\mu_{a^\star})\ge \frac{f(t)}{N_{a^\star}(t)} + \frac{\epsilon^2}{2}\right\} \\
&\subseteq \cup_{n=1}^{t-K+1} \left\{\Kinf(\hat{\nu}_{a^\star,n},\mu_{a^\star})\ge \frac{f(t)}{n}+\frac{\epsilon^2}{2}\right\}.
\end{align*}
The remainder of the proof is now the same as in \cite{Cappeetal2013b}. In particular, their Equation 26 combined with the bounds given after their Equation 35 shows that the right-hand side of \eqref{eq:PmudaggtUastar} is upper bounded by $36\mu_{a^\star}^{-4}\left(2 + \log\log T\right)\left(\log T\right)^{4/5}=o(\log T)$. %Again we note that their Equation 26 holds despite the union above being over $n=\ell,\ldots,t-K+\ell$ rather than $n=1,\ldots,t-K+1$ as in \cite{Cappeetal2013b}.
\qed\end{proof}

\begin{proof}[Proof of Lemma~\ref{lem:TermA} for KL-UCB in the settings of Theorems \ref{thm:expfam} and \ref{thm:finsup}]
Fix $a\in\underline{\cU}\cup\overline{\cU}$. For ease of notation, we analyze $\E[M_a^{a^\star}(T)]$ rather than $\E[M_a^{a^\star}(T^{(G)})]$, but for fixed $G<\infty$ there is no loss of generality in doing so. If $a\in\underline{\cU}$, then let $\mu^{\dagger}= \cc_a\rho_{a^\star}$, and otherwise, fix $\mu^\dagger\in(\mu_a,\mu_{+})$. Let $\rho^\dagger\equiv\mu^\dagger/\cc_a$ (setting of Theorem~\ref{thm:expfam}) or $\rho^\dagger\equiv\left[1-\log(T)^{-1/5}\right]\mu^\dagger/\cc_a$ (setting of Theorem~\ref{thm:finsup}). Note that $\rho^\dagger<\mu_{+}/\cc_a$. Analogous arguments to those used for \eqref{eq:Atp1a} show that
\begin{align}
&\left\{a\in\A(t+1), \frac{U_{a^\star}(t)}{\cc_{a^\star}}<\hat{\rho}^\star(t)\right\} \nonumber \\
&\subseteq\left\{a\in\A(t+1), \frac{U_{a^\star}(t)}{\cc_{a^\star}}<\hat{\rho}^\star(t),\rho^\dagger\ge \frac{U_{a}(t)}{\cc_a}\right\}\cup\left\{a\in\A(t+1), \frac{U_{a^\star}(t)}{\cc_{a^\star}}<\hat{\rho}^\star(t), \rho^\dagger<\frac{U_{a}(t)}{\cc_a}\right\} \nonumber \\
&\subseteq \left\{\rho^\dagger\ge \frac{U_{a^\star}(t)}{\cc_{a^\star}}\right\}\cup \left\{a\in\A(t+1), \rho^\dagger<\frac{U_a(t)}{\cc_a}\right\}. \label{eq:Atp1atilde}
\end{align}
Let
\begin{align*}
b_a^{a^\star}(T)&\equiv \left\lceil\frac{f(T)}{\Kinf(\nu_a,\mu^\dagger)}\right\rceil.
\end{align*}
Similarly to \eqref{eq:T1T2}, we have that
\begin{align*}
\E[M_a^{a^\star}(T)]&\le \frac{f(T)}{\Kinf(\nu_a,\mu_{a^\star})} + \sum_{n=b_a^{a^\star}(T) + 1}^{\infty} \Prob\left\{\hat{\nu}_{a,n}\in\mathcal{C}_{\cc_a\rho^\dagger,f(T)/n}\right\} + \sum_{t=K}^{T-1} \Prob\left\{\rho^\dagger\ge \frac{U_{a^\star}(t)}{\cc_{a^\star}}\right\} + 2.
\end{align*}
By Lemmas \ref{lem:term1} and \ref{lem:term2astar},
\begin{align}
\E[M_a^{a^\star}(T)]\le \frac{\log T}{\Kinf(\nu_a,\mu^\dagger)} + o(\log T). \label{eq:TnotTG}
\end{align}
In what follows we refer to this $o(\log T)$ term as $r(T,\mu^\dagger)$, where we note that $r(T,\mu^\dagger)/\log T\rightarrow 0$ for each fixed $\mu^\dagger\in(\mu_a,\mu_{+})$. If $a\in\overline{\cU}$, we will obtain our result by letting $\mu^\dagger\rightarrow \mu_{+}$. Thus, there exists a sequence $\mu^\dagger(T)\rightarrow\mu_{+}$ such that $r(T,\mu^\dagger(T))/\log T\rightarrow 0$. Noting $\liminf_{\mu^\dagger\rightarrow\mu_{+}} \Kinf(\nu_a,\mu^\dagger) = +\infty$ in the setting of both theorems, we see that
\begin{align*}
\E[M_a^{a^\star}(T)]\le \frac{\log T}{\Kinf(\nu_a,\mu^\dagger(T))} + r(T,\mu^\dagger(T)) = o(\log T).
\end{align*}
This is the desired result when $a\in\overline{\cU}$. If, instead, $a\in\underline{\cU}$, then replacing $T$ by $T^{(G)}$ in \eqref{eq:TnotTG} (for $T$ large enough so that $T^{(G)}>1$), and recalling that $\mu^\dagger=\cc_a \rho_{a^\star}$ when $a\in\underline{\cU}$, gives the desired result.
\qed\end{proof}

\begin{proof}[Proof of Lemma~\ref{lem:TermB} for KL-UCB in the settings of Theorems \ref{thm:expfam} and \ref{thm:finsup}]
Fix $g\in\mathbb{N}$, $a\in\underline{\cU}\subset\cM\cup\cN$, and $T^{(g)}$ such that $T^{(g)}>1$. In the setting of Theorem~\ref{thm:expfam} let $\rho^\dagger=\rho_{a^\star}$, and in the setting of Theorem~\ref{thm:finsup} let $\rho^\dagger=[1-\log(T)^{-1/5}]\rho_{a^\star}$. By \eqref{eq:Atp1atilde} and the fact that $\{\rho^\dagger<U_a(t)/\cc_a\} = \left\{\hat{\nu}_{a,N_a(t)}\in\mathcal{C}_{\cc_a \rho^\dagger,f(t)/N_a(t)}\right\}$,
\begin{align*}
&\E[M_a^{a^\star}(T^{(g-1)})-M_a^{a^\star}(T^{(g)})] \\
&\le \sum_{t=T^{(g)}}^{T^{(g-1)}-1} \Prob\left\{\rho^\dagger\ge \frac{U_{a^\star}(t)}{\cc_{a^\star}}\right\} + \sum_{t=T^{(g)}}^{T^{(g-1)}-1} \Prob\left\{a\in\A(t+1), \hat{\nu}_{a,N_a(t)}\in\mathcal{C}_{\cc_a \rho^\dagger,f(t)/N_a(t)}\right\}.
\end{align*}
The first term in the right hand side is upper bounded by the same sum from $t=K$ to $T-1$, and is thus $o(\log T)$ by Lemma~\ref{lem:term2astar}. For the second term, let $b_a'(T,g)\equiv \lceil(1-\delta)\frac{f(T^{(g)})}{\Kinf(\nu_a,\cc_a \rho^\star)}\rceil$ if $a\in\cN$ and let $b_a'(T,g)\equiv \lceil\frac{f(T^{(g)})}{(1-\delta)\Kinf(\nu_a,\cc_a \rho_{a^\star})}\rceil$ if $a\in\cM$. Similar arguments to those used to derive \eqref{eq:taunKL} in Section~\ref{sec:suboptrare} show that, for $T$ large enough so that $T^{(g)}\ge K$,
\begin{align*}
\sum_{t=T^{(g)}}^{T^{(g-1)}-1} &\Prob\left\{a\in\A(t+1), \hat{\nu}_{a,N_a(t)}\in\mathcal{C}_{\cc_a \rho^\dagger,f(t)/N_a(t)}\right\} \\
\le&\, \sum_{n=1}^{T^{(g-1)}-K}\;\sum_{t=T^{(g)}}^{T^{(g-1)}-1} \Prob\left\{\hat{\nu}_{a,n}\in\mathcal{C}_{\cc_a \rho^\dagger,f(T^{(g-1)})/n}, \tau_{a,n+1}=t+1\right\}.
\end{align*}
We split the sum over $n$ into a sum $S_1$ from $n=1$ to $b_a'(T,g)$ and a sum $S_2$ from $n=b_a'(T,g)+1$ to $T^{(g-1)}-K$. For the latter sum, the fact that, for each $n$, $\tau_{a,n+1}=t+1$ for at most one $t$ in a given interval, yields that
\begin{align*}
S_2&\le \sum_{n= b_a'(T,g)+1}^{T^{(g-1)}-K} \Prob\left\{\hat{\nu}_{a,n}\in\mathcal{C}_{\mu^\dagger,f(T^{(g-1)})/n}\right\}.
\end{align*}
If $a\in\cN$, then $\delta$ satisfying \eqref{eq:deltalogic} yields that $b_a'(T,g)> \frac{f(T^{(g-1)})}{\Kinf(\nu_a,\cc_a \rho_{a^\star})}$, and so the above sum is $o(\log T)$ by Lemma~\ref{lem:term1}. If $a\in\cM$, then $b_a'(T,g)=\lceil\frac{f(T^{(g-1)})}{\Kinf(\nu_a,\cc_a \rho_{a^\star})}\rceil$, and so again the above sum is $o(\log T)$.

We now bound $S_1$. Note that if $N_a(T^{(g)}-1)>b_a'(T,g)$, then, for every $n\le b_a'(T,g)$, $\tau_{a,n+1}< T^{(g)}$ and $S_1=0$ (the sum over $t$ is void). Therefore,
\begin{align*}
S_1&\le \sum_{n=1}^{b_a'(T,g)} \Prob\left\{N_a(T^{(g)}-1)\le b_a'(T,g)\right\} = b_a'(T,g)\Prob\left\{N_a(T^{(g)}-1)\le b_a'(T,g)\right\}.
\end{align*}
From Lemma~\ref{lem:cons}, KL-UCB is uniformly efficient. Thus, by  \eqref{eq:problb}, for any $a\in\cM$ one has 
\[\lim_{T\rightarrow \infty} \Prob\left(N_a(T^{(g)}-1) \leq \frac{2\log(T^{(g)})}{(1-\delta)\Kinf(\nu_a,\cc_a \rho_{a^\star})}\right) = 0,\]
where we use the fact that $\Kinf(\nu_a,\cc_a \rho^\star)=0$ and choose $\epsilon =(1-\delta)\Kinf(\nu_a,\cc_a \rho_{a^\star})/2 >0$. This yields that $\Prob\left\{N_a\left(T^{(g)}-1\right)< b_a'(T,g)\right\}\rightarrow 0$ as $T\rightarrow\infty$ and $S_1=o(\log T)$.

If $a\in\cN$, then Lemma~\ref{lem:cons} and \eqref{eq:problb} from Lemma~\ref{lem:lb} yield that
\begin{align*}
\Prob\left\{N_a\left(T^{(g)}-1\right)< (1-\delta)\frac{\log T^{(g)}}{\Kinf(\nu_a,\cc_a \rho^\star)}\right\}\rightarrow 0\textnormal{ as $T\rightarrow 0$.}
\end{align*}
The fact that $\lim_T f(T)/\log T= 1$ shows that $b_a'(T,g)=(1-\delta)\frac{\log T^{(g)}}{\Kinf(\nu_a,\cc_a \rho^\star)} + o(\log T)$. Plugging this into \eqref{eq:problb} from Lemma~\ref{lem:lb} (which holds for \textit{every} $\delta$ between $0$ and $1$) yields that $\Prob\left\{N_a\left(T^{(g)}-1\right)< b_a'(T,g)\right\}\rightarrow 0$ as $T\rightarrow\infty$. It follows that $S_1=o(\log T)$.

We have then shown that $\E[M_a^{a^\star}(T^{(g-1)})-M_a^{a^\star}(T^{(g)})]= o(\log T)$ for each $a\in\underline{\cU}\subset\cM\cup\cN$ and each $g\le G$. As Term B is a sum of finitely many such terms, Term B is $o(\log T)$.
\qed\end{proof}

\newpage

\section{Supplementary Proofs for Thompson Sampling} \label{app:thomproof}
We begin with a lemma.
\begin{lemma} \label{lem:betabinomial}
For any fixed real number $L$, arm $a$, $\mu_a<\mu^\dagger<\theta^\dagger$, and $t\ge 1$,
\begin{align*}
I\left\{\hat{\mu}_a(t)\le \mu^\dagger,N_a(t)\ge L\right\}\Prob\left\{\theta_a(t)>\theta^\dagger\middle|\mathcal{F}(t)\right\}\le e^{-(L+1)\KL(\mu^\dagger,\theta^\dagger)}.
\end{align*}
\end{lemma}
\begin{proof}
From Fact 3 in \cite{Agrawal&Goyal2012} \citep[also used in][]{Agrawal&Goyal2011,Kaufmannetal2012b,Kaufmannetal2012},
\begin{align*}
\Prob\left(\left.\theta_a(t)>\theta^\dagger\right|\mathcal{F}(t)\right)&= \Prob\left(\left.\sum_{n=1}^{N_a(T)+1} Z_n\le \sum_{n=1}^{N_a(T)} \Ind\{X_{a,n}=0\} \right|\mathcal{F}(t)\right),
\end{align*}
where $\{Z_n\}$ is an i.i.d. sequence (independent of all other quantities under consideration) of Bernoulli random variables with mean $\theta^\dagger$. Upper bounding the right-hand side yields
\begin{align*}
\Prob\left(\left.\theta_a(t)>\theta^\dagger\right|\mathcal{F}(t)\right)&\le \Prob\left(\left.\frac{1}{N_a(T)+1}\sum_{n=1}^{N_a(T)+1} Z_n\le \hat{\mu}_a(T) \right|\mathcal{F}(t)\right).
\end{align*}
Using that $\mu^\dagger<\theta^\dagger$, the Chernoff-Hoeffding bound gives that $\Prob\left(\left.\theta_a(t)>\theta^\dagger\right|\mathcal{F}(t)\right)$ is no larger than $e^{-[N_a(t)+1]\KL(\hat{\mu}_a(t),\theta^\dagger)}$. Multiplying the left-hand side by $I\left\{\hat{\mu}_a(t)\le \mu^\dagger,N_a(t)\ge L\right\}$, this yields the upper bound $e^{-(L+1)\KL(\mu^\dagger,\theta^\dagger)}$.
\qed\end{proof}

\begin{proof}[Proof of Lemma~\ref{lem:suboptasopt}]
Let $\tilde{\theta}_{\tilde{a}}(t)=\theta_{\tilde{a}}(t)$ for all $\tilde{a}\not=a^\star$ and let $\tilde{\theta}_{a^\star}(t)=-\infty$. Define the event $B\equiv \left\{\rho^\star(\tilde{\theta}_a(t) : a=1,\ldots,K+1)<\theta_a(t)/\cc_a\le \rho^\ddagger\right\}$. Observe that
\begin{align}
\Prob&\left\{\hat{\rho}^\star(t)\le \frac{\theta_a(t)}{\cc_a}\le \rho^\ddagger,\frac{\theta_{a^\star}(t)}{\cc_{a^\star}}< \hat{\rho}^\star(t)\middle| \mathcal{F}(t)\right\} \nonumber \\
&= \Prob\left(\left\{\hat{\rho}^\star(t)\le \frac{\theta_a(t)}{\cc_a},\frac{\theta_{a^\star}(t)}{\cc_{a^\star}}< \hat{\rho}^\star(t)\right\}\cap B\middle| \mathcal{F}(t)\right) \nonumber \\
&\le \Prob\left(\left\{\frac{\theta_{a^\star}(t)}{\cc_{a^\star}}\le\rho^\ddagger\right\}\cap B\middle| \mathcal{F}(t)\right). \label{eq:firstordering}
\end{align}
The event $\{\theta_{a^\star}(t)/\cc_{a^\star}> \rho^\ddagger\}$ is independent of the event $B$ conditional on $\mathcal{F}(t)$, and so the fact that $\left\{\theta_{a^\star}(t)/\cc_{a^\star}> \rho^\ddagger\right\}\cap B\subseteq\left\{\theta_{a^\star}(t)/\cc_{a^\star}\ge \hat{\rho}^{\star}(t)\right\}$ yields
\begin{align*}
\Prob(B|\mathcal{F}(t))\le \frac{\Prob\left(\left.\theta_{a^\star}(t)/\cc_{a^\star}\ge \hat{\rho}^{\star}(t)\right|\mathcal{F}(t)\right)}{\Prob\left(\left.\theta_{a^\star}(t)/\cc_{a^\star}> \rho^\ddagger\right|\mathcal{F}(t)\right)}.
\end{align*}
We note that $\Prob\left(\left.\theta_{a^\star}(t)>\cc_{a^\star}\rho^\ddagger\right|\mathcal{F}(t)\right)$ is positive (a beta distribution with at least one success is larger than $\cc_{a^\star}\rho^\ddagger<1$ with positive probability). Finally, since $a\in\A(t+1)$ implies that $\theta_{a^\star}(t)/\cc_{a^\star}\ge \hat{\rho}^{\star}(t)$, \eqref{eq:firstordering} yields
\begin{align*}
\Prob&\left(\left.a\in\A(t+1),\theta_a(t)/\cc_a\le \rho^\ddagger,\theta_{a^\star}(t)/\cc_{a^\star}< \hat{\rho}^\star(t)\right|\mathcal{F}(t)\right) \\
&\le \Prob\left(\left\{\theta_{a^\star}(t)/\cc_{a^\star}\le\rho^\ddagger\right\}\cap B\middle| \mathcal{F}(t)\right) \\
&= \Prob\left(\theta_{a^\star}(t)/\cc_{a^\star}\le\rho^\ddagger\middle| \mathcal{F}(t)\right)\Prob\left(\left.B\right|\mathcal{F}(t)\right) \\
&\le \Prob\left(\theta_{a^\star}(t)/\cc_{a^\star}\le\rho^\ddagger\middle| \mathcal{F}(t)\right)\frac{\Prob\left(\left.\theta_{a^\star}(t)/\cc_{a^\star}\ge \hat{\rho}^{\star}(t)\right|\mathcal{F}(t)\right)}{\Prob\left(\left.\theta_{a^\star}(t)/\cc_{a^\star}> \rho^\ddagger\right|\mathcal{F}(t)\right)}.
\end{align*}
\qed\end{proof}

\begin{proof}[Proof of Lemma~\ref{lem:thomtransferTtoN}]
Using \eqref{eq:calSvscalA}, one can write
\begin{align*}
\E&\left[\sum_{t=0}^{T-1} \frac{1-p_{a^\star}^{\rho^\ddagger}(t)}{p_{a^\star}^{\rho^\ddagger}(t)}\Prob\left(\frac{\theta_{a^\star}(t)}{\cc_{a^\star}}\ge \hat{\rho}^{\star}(t)\middle|\mathcal{F}(t)\right)\right] \\
&\le \mathring{q}_{a^\star}^{-1} \E\left[\sum_{t=0}^{T-1} \frac{1-p_{a^\star}^{\rho^\ddagger}(t)}{p_{a^\star}^{\rho^\ddagger}(t)}\Prob\left(a^\star\in\A(t+1)\middle|\mathcal{F}(t)\right)\right] \\
&= \mathring{q}_{a^\star}^{-1} \E\left[\sum_{t=0}^{T-1} \frac{1-p_{a^\star,N_{a^\star}(t)}}{p_{a^\star,N_{a^\star}(t)}}\Ind\left\{a^\star\in\A(t+1)\right\}\right] \\
&= \mathring{q}_{a^\star}^{-1} \E\left[\sum_{t=0}^{T-1} \sum_{n=0}^{T-1} \frac{1-p_{a^\star,n}^{\rho^\ddagger}}{p_{a^\star,n}^{\rho^\ddagger}}\Ind\left\{\tau_{a^\star,n+1}=t+1\right\}\right] \\
&\le \mathring{q}_{a^\star}^{-1} \E\left[\sum_{n=0}^{T-1} \frac{1-p_{a^\star,n}^{\rho^\ddagger}}{p_{a^\star,n}^{\rho^\ddagger}}\right],
\end{align*}
where the latter inequality holds because $\tau_{a^\star,n+1}=t+1$ for at most one $t$ in $\{0,\ldots,T-1\}$.
\qed\end{proof}

\begin{proof}[Proof of Lemma~\ref{lem:termiii}]
Let $L^\dagger(T)\equiv \frac{\log T}{\KL(\cc_a \rho^\dagger,\cc_a \rho^\ddagger)}$. We have that
\begin{align}
\sum_{t=0}^{T-1}& \Prob\left\{a\in\A(t+1),\theta_a(t)>\cc_a \rho^\ddagger,\hat{\mu}_a(t)\le\cc_a \rho^\dagger\right\} \nonumber \\
=&\, \E\left[\sum_{t=0}^{T-1} \Ind\left\{N_a(t)< L^\dagger(T)-1,\hat{\mu}_a(t)\le\cc_a \rho^\dagger\right\}\Prob\left(\left.a\in\A(t+1),\theta_a(t)>\cc_a \rho^\ddagger\right|\mathcal{F}(t)\right)\right] \nonumber \\
&+ \E\left[\sum_{t=0}^{T-1} \Ind\left\{N_a(t)\ge L^\dagger(T)-1,\hat{\mu}_a(t)\le\cc_a \rho^\dagger\right\}\Prob\left(\left.a\in\A(t+1),\theta_a(t)>\cc_a \rho^\ddagger\right|\mathcal{F}(t)\right)\right] \nonumber \\
\le&\, \E\left[\sum_{t=0}^{T-1} \Ind\left\{N_a(t)< L^\dagger(T)-1\right\}\Prob\left(\left.a\in\A(t+1)\right|\mathcal{F}(t)\right)\right] \nonumber \\
&+ \E\left[\sum_{t=0}^{T-1} \Ind\left\{N_a(t)\ge L^\dagger(T)-1,\hat{\mu}_a(t)\le\cc_a \rho^\dagger\right\}\Prob\left(\left.\theta_a(t)>\cc_a \rho^\ddagger\right|\mathcal{F}(t)\right)\right]. \label{eq:thompsconstbd}
\end{align}
The first term in the right hand side equals $\E\left[\sum_{t=0}^{T-1} \Ind\left\{N_a(t)< L^\dagger(T)-1,a\in\A(t+1)\right\}\right]$. Hence it is no larger than $L^\dagger(T)-1$ (the sum has at most $L^\dagger(T)-1$ nonzero terms). For the second term, Lemma~\ref{lem:betabinomial} yields
\begin{align*}
&\Ind\left\{\hat{\mu}_a(t)<\cc_a \rho^\dagger,N_a(T)\ge L^\dagger(T)-1\right\}\Prob\left(\left.\theta_a(t)>\cc_a \rho^\ddagger\right|\mathcal{F}(t)\right)\le e^{-L^\dagger(T)\KL(\cc_a \rho^\dagger,\cc_a \rho^\ddagger)}= T^{-1}.
\end{align*}
It follows that the second term on the right of \eqref{eq:thompsconstbd} is upper bounded by $\sum_{t=0}^{T-1} T^{-1} = 1$. This completes the proof.
\qed\end{proof}

\begin{proof}[Proof of Lemma~\ref{lem:TermA} for Thompson sampling in the setting of Theorem~\ref{thm:thom}]
Fix $a\in\underline{\cU}\cup\overline{\cU}$ and $\epsilon\in(0,1)$. For ease of notation, we analyze $\E[M_a^{a^\star}(T)]$ rather than $\E[M_a^{a^\star}(T^{(G)})]$, but for fixed $G<\infty$ there is no loss of generality in doing so. If $a\in\underline{\cU}$, then let $\mu^\dagger= \cc_a\rho_{a^\star}$, and otherwise, fix $\mu^\dagger\in(\mu_a,\mu_{+})$. Let $\rho^\dagger$ and $\rho^\ddagger$ satisfy $\rho_a<\rho^\dagger<\rho^\ddagger<\mu^\dagger/\cc_a$ (exact quantities to be specified at the end of the proof). Note that
\begin{align*}
&\left\{a\in\A(t+1), \frac{\theta_{a^\star}(t)}{\cc_{a^\star}}<\hat{\rho}^\star(t)\right\} \\
&\subseteq \left\{a\in\A(t+1),\frac{\theta_a(t)}{\cc_a}\le\rho^\ddagger, \frac{\theta_{a^\star}(t)}{\cc_{a^\star}}<\hat{\rho}^\star(t)\right\}\cup \left\{a\in\A(t+1), \frac{\theta_a(t)}{\cc_a}>\rho^\ddagger\right\}.
\end{align*}
Recalling that $\E[M_a^{a^\star}(T)]$ is equal to $\sum_{t=0}^{T-1} \Prob\left\{a\in\A(t+1), \theta_{a^\star}(t)/\cc_{a^\star}<\hat{\rho}^\star(t)\right\}$, the above yields
\begin{align}
\E[M_a^{a^\star}(T)]\le&\, \sum_{t=0}^{T-1}\Prob\left\{a\in\A(t+1),\frac{\theta_a(t)}{\cc_a}\le\rho^\ddagger, \frac{\theta_{a^\star}(t)}{\cc_{a^\star}}<\hat{\rho}^\star(t)\right\} \nonumber \\
&+ \sum_{t=0}^{T-1} \Prob\left\{a\in\A(t+1), \frac{\hat{\mu}_a(t)}{\cc_a}>\rho^\dagger\right\} \nonumber \\
&+ \sum_{t=0}^{T-1} \Prob\left\{a\in\A(t+1), \frac{\theta_a(t)}{\cc_a}>\rho^\ddagger,\frac{\hat{\mu}_a(t)}{\cc_a}\le\rho^\dagger\right\}. \label{eq:termsitoiiiagain}
\end{align}
Note that the right-hand side of the above is almost identical to \eqref{eq:termsitoiii}. Note that all of the results used to control the three terms on the right-hand side of \eqref{eq:termsitoiii} hold for any $a$ with $\rho_a\le\rho^\star$ provided $\rho_a<\rho^\dagger<\rho^\ddagger<\mu^\dagger/\cc_a$. In particular, we are referring to Lemma~\ref{lem:suboptasopt}, \eqref{eq:suboptasoptsum}, Lemma~\ref{lem:carefulbinomialbound}, \eqref{eq:termii}, Lemma~\ref{lem:termii}, and Lemma~\ref{lem:termiii}. Hence, $\E[M_a^{a^\star}(T)]\le \frac{\log T}{\KL(\cc_a \rho^\dagger,\cc_a \rho^\ddagger)} + o(\log T)$.

Selecting $\rho^\dagger$ and $\rho^\ddagger$ as in the proof of \eqref{eq:subopt} and \eqref{eq:suboptindifference} from Theorem~\ref{thm:thom} yields $\E[M_a^{a^\star}(T)]\le (1+\epsilon)^2\frac{\log T}{\KL(\mu_a,\mu^\dagger)} + o(\log T)$. As $\epsilon$ was arbitrary, dividing both sides by $\log T$ and taking $T\rightarrow\infty$ followed by $\epsilon\rightarrow 0$ yields that $\E[M_a^{a^\star}(T)]\le \frac{\log T}{\KL(\mu_a,\mu^\dagger)} + o(\log T)$. If $a\in\underline{\cU}$, then replacing $T$ by $T^{(G)}$ (for $T$ large enough so that $T^{(G)}>1$) gives the desired $\E[M_a^{a^\star}(T^{(G)})]\le (1-\delta)^G\frac{\log T}{\Kinf(\nu_a,\mu^\dagger)} + o(\log T)$ in light of the fact that $\mu^\dagger=\cc_a\rho_{a^\star}$. If, on the other hand, $a\in\overline{\cU}$, then the same arguments used to conclude the $a\in\overline{\cU}$ result in the proof of Lemma~\ref{lem:TermA} for KL-UCB, namely selecting an appropriate sequence $\mu^\dagger(T)\rightarrow\mu_{+}$, can be used to show that $\E[M_a^{a^\star}(T^{(G)})]=o(\log T)$.
\qed\end{proof}

\begin{proof}[Proof of Lemma~\ref{lem:TermB} for Thompson sampling in the setting of Theorem~\ref{thm:thom}]
Fix $g\in\mathbb{N}$, an arm $a\in\underline{\cU}\subset\cM\cup\cN$, and $T^{(g)}$ such that $T^{(g)}>1$. Let $\rho^\dagger$ and $\rho^\ddagger$ satisfy $\rho_a<\rho^\dagger<\rho^\ddagger<\rho_{a^\star}$ and $\KL(\cc_a \rho^\dagger,\cc_a \rho^\ddagger)\ge (1-\delta)\KL(\mu_a,\cc_a \rho_{a^\star})$. By the same arguments used for \eqref{eq:termsitoiiiagain},
\begin{align}
\E&[M_a^{a^\star}(T^{(g-1)})-M_a^{a^\star}(T^{(g)})] \nonumber \\
\le&\, \sum_{t=T^{(g)}}^{T^{(g-1)}-1}\Prob\left\{a\in\A(t+1),\frac{\theta_a(t)}{\cc_a}\le\rho^\ddagger,\frac{\theta_{a^\star}(t)}{\cc_{a^\star}} < \hat{\rho}^\star(t)\right\} \nonumber \\
&+ \sum_{t=T^{(g)}}^{T^{(g-1)}-1} \Prob\left\{a\in\A(t+1), \frac{\hat{\mu}_a(t)}{\cc_a}>\rho^\dagger\right\} \nonumber \\
&+ \sum_{t=T^{(g)}}^{T^{(g-1)}-1} \Prob\left\{a\in\A(t+1), \frac{\theta_a(t)}{\cc_a}>\rho^\ddagger,\frac{\hat{\mu}_a(t)}{\cc_a}\le\rho^\dagger\right\}. \label{eq:thomNdiff}
\end{align}
The first two sums are trivially upper bounded by the sums from $t=0$ to $T-1$, and thus are $o(\log T)$ by Lemma~\ref{lem:suboptasopt}, \eqref{eq:suboptasoptsum}, Lemma~\ref{lem:carefulbinomialbound}, \eqref{eq:termii}, and Lemma~\ref{lem:termii}. If $a\in\cN$, then let $b_a(T,g)\equiv(1-\delta)\frac{\log T^{(g)}}{\KL(\mu_a,\cc_a \rho^\star)}$, and if $a\in\cM$ then let $b_a(T,g)\equiv\frac{\log T^{(g)}}{(1-\delta)\KL(\mu_a,\cc_a \rho_{a^\star})}$. We have that
\begin{align}
&\sum_{t=T^{(g)}}^{T^{(g-1)}-1} \Prob\left\{a\in\A(t+1), \frac{\theta_a(t)}{\cc_a}>\rho^\ddagger,\frac{\hat{\mu}_a(t)}{\cc_a}\le\rho^\dagger\right\} \nonumber \\
=&\, \E\left[\sum_{t=T^{(g)}}^{T^{(g-1)}-1} \Ind\left\{\frac{\hat{\mu}_a(t)}{\cc_a}\le\rho^\dagger,N_a(t)\ge b_a(T,g)\right\}\Prob\left\{a\in\A(t+1), \frac{\theta_a(t)}{\cc_a}>\rho^\ddagger\middle|\mathcal{F}(t)\right\}\right] \nonumber \\
&+ \E\left[\sum_{t=T^{(g)}}^{T^{(g-1)}-1} \Ind\left\{\frac{\hat{\mu}_a(t)}{\cc_a}\le\rho^\dagger,N_a(t)< b_a(T,g)\right\}\Prob\left\{a\in\A(t+1), \frac{\theta_a(t)}{\cc_a}>\rho^\ddagger\middle|\mathcal{F}(t)\right\}\right]. \label{eq:thomNdecomp}
\end{align}
If $a\in\cN$, then Lemma~\ref{lem:betabinomial} and $\KL(\cc_a\rho^\dagger,\cc_a\rho^\ddagger)\ge (1-\delta)\KL(\mu_a,\cc_a \rho_{a^\star})$ yield that the first term on the right is upper bounded by
\begin{align*}
\sum_{t=T^{(g)}}^{T^{(g-1)}-1} &\exp\left[-(1-\delta)^2\frac{\log T^{(g-1)}}{\KL(\mu_a,\cc_a \rho^\star)}\KL(\mu_a,\cc_a \rho_{a^\star})\right] \\
&\le T^{(g-1)}\exp\left[-(1-\delta)^2\frac{\log T^{(g-1)}}{\KL(\mu_a,\cc_a \rho^\star)}\KL(\mu_a,\cc_a \rho_{a^\star})\right]\le 1,
\end{align*}
where the second inequality holds because $\delta$ satisfies \eqref{eq:deltalogic}. If $a\in\cM$, then we instead have that this term is no larger than
\begin{align*}
\sum_{t=T^{(g)}}^{T^{(g-1)}-1} \exp\left[-\frac{\log T^{(g-1)}}{\KL(\mu_a,\cc_a \rho_{a^\star})}\KL(\mu_a,\cc_a \rho_{a^\star})\right]\le 1.
\end{align*}
For the second term in \eqref{eq:thomNdecomp}, note that
\begin{align*}
\E&\left[\sum_{t=T^{(g)}}^{T^{(g-1)}-1} \Ind\left\{\frac{\hat{\mu}_a(t)}{\cc_a}\le\rho^\dagger,N_a(t)< b_a(T,g)\right\}\Prob\left\{a\in\A(t+1), \frac{\theta_a(t)}{\cc_a}>\rho^\ddagger\middle|\mathcal{F}(t)\right\}\right] \\
&\le \E\left[\sum_{t=T^{(g)}}^{T^{(g-1)}-1} \Ind\left\{N_a(t)< b_a(T,g)\right\}\Prob\left\{a\in\A(t+1)\middle|\mathcal{F}(t)\right\}\right] \\
&= \E\left[\sum_{t=T^{(g)}}^{T^{(g-1)}-1} \Ind\left\{N_a(t)< b_a(T,g),a\in\A(t+1)\right\}\right] \\
&= \E\left[\Ind\left\{N_a\left(T^{(g)}\right)< b_a(T,g)\right\}\sum_{t=T^{(g)}}^{T^{(g-1)}-1} \Ind\left\{N_a(t)< b_a(T,g),a\in\A(t+1)\right\}\right] \\
&\le b_a(T,g)\Prob\left\{N_a\left(T^{(g)}\right)< b_a(T,g)\right\},
\end{align*}
where the final inequality uses that the sum inside the expectation is at most $b_a(T,g)$. By the uniform efficiency of the algorithm established in Lemma~\ref{lem:cons} and \eqref{eq:problb} from Lemma~\ref{lem:lb}, the probability in the final inequality is $o(1)$, and thus the above is $o(b_a(T,g))=o(\log T)$. Thus \eqref{eq:thomNdecomp} is $o(\log T)$.

Plugging this into \eqref{eq:thomNdiff} yields that $\E[M_a^{a^\star}(T^{(g-1)})-M_a^{a^\star}(T^{(g)})]=o(\log T)$ for each $a\in\underline{\cU}\subset\cM\cup\cN$ and each $g\le G$. As Term B is a sum of finitely many such terms, Term B is $o(\log T)$. 
\qed\end{proof}

\end{document}